\documentclass[twoside]{article}

\usepackage[accepted]{home_aistats2022}

\usepackage{hyperref}       
\usepackage{url}            
\usepackage{microtype}      
\usepackage{amsfonts}
\usepackage{amsthm}
\usepackage{amsmath}
\usepackage{amssymb}
\usepackage{mdframed}
\usepackage{tikz}
\usepackage{bbold}
\usepackage[]{apacite} 
\usepackage{float}
\allowdisplaybreaks

\newcommand{\sca}[2]{\langle #1 | #2\rangle}
\newcommand{\nr}[1]{\left\Vert #1\right\Vert}
\newcommand{\abs}[1]{\left\vert #1\right\vert}
\newcommand{\Nsp}{\mathbb{N}}
\newcommand{\Rsp}{\mathbb{R}}

\newcommand{\Class}{\mathcal{C}}
\newcommand{\Wass}{\mathrm{W}}

\newcommand{\X}{\mathcal{X}}
\newcommand{\Y}{\mathcal{Y}}

\newcommand{\dd}{\mathrm{d}}

\renewcommand{\L}{\mathrm{L}}

\newcommand{\diam}{\mathrm{diam}}
\newcommand{\diag}{\mathrm{diag}}
\newcommand{\eps}{\varepsilon}

\newcommand{\Lag}{\mathrm{Lag}}
\newcommand{\KL}{\mathrm{KL}}
\newcommand{\Prob}{\mathcal{P}}

\newcommand{\Var}{\mathbb{V}\mathrm{ar}}

\newcommand{\Esp}{\mathbb{E}}
\newcommand{\Kant}{\mathcal{K}}

\newcommand{\Li}{\mathrm{Li}}

\newmdtheoremenv{theo}{Theorem}
\newmdtheoremenv{coro}{Corollary}

\newtheorem{theorem}{Theorem}[section]

\newtheorem{corollary}[theorem]{Corollary}

\newtheorem{proposition}[theorem]{Proposition}
\newtheorem{assumption}{Assumption}

\theoremstyle{definition}

\newtheorem{remark}{Remark}[section]

\begin{document}

%

%

\twocolumn[
\aistatstitle{Nearly Tight Convergence Bounds for Semi-discrete Entropic Optimal Transport}

\aistatsauthor{ Alex Delalande }

\aistatsaddress{ Université Paris-Saclay \& Inria Saclay }
]

\begin{abstract}
  We derive nearly tight and non-asymptotic convergence bounds for solutions of entropic semi-discrete optimal transport. These bounds quantify the stability of the dual solutions of the regularized problem (sometimes called Sinkhorn potentials) w.r.t. the regularization parameter, for which we ensure a better than Lipschitz dependence. Such facts may be a first step towards a mathematical justification of annealing or $\eps$-scaling heuristics for the numerical resolution of regularized semi-discrete optimal transport. Our results also entail a non-asymptotic and tight expansion of the difference between the entropic and the unregularized costs.
\end{abstract}

\section{INTRODUCTION}



Optimal transport and the distances it defines \cite{villani2008optimal} are now widely acknowledged as important tools for machine learning \shortcite{ML_3, ML_2, ML_6, ML_5, ML_1, ML_4} and statistics \cite{stat3, stat2, stat1, stat4}. In these fields, it is also recognized that the original formulation of the transport problem suffers in general from poor computationability and statistical behavior with respect to the dimension, and that some form of regularization can be helpful. In this state of mind, the entropic regularization of the optimal transport problem, which dates back to \cite{schrodinger} and was revisited in the recent years by \cite{Cuturi:2013}, has proven to be a relevant choice of regularization. For two compact subsets $\X,\Y$ of $\Rsp^d$, two probability measures $\rho \in \Prob(\X), \mu \in \Prob(\Y)$, and for $\eps \geq 0$, the quadratic optimal transport problem between $\rho$ and $\mu$ with entropic regularization of parameter $\eps$ reads
\begin{equation}
\label{eq:quadratic-ot}
    \min_{\pi \in \Pi(\rho, \mu)} \int_{\X \times \Y} \nr{x - y}^2 \dd \pi(x,y) + \eps \KL(\pi \vert \rho \otimes \mu ), \tag{$\mathrm{P}_\eps$}
\end{equation}
where $\Pi(\rho, \mu)$ denotes the set of couplings between $\rho$ and $\mu$ and $\KL$ denotes the Kullback-Leibler divergence or relative entropy (up to an additive term):
\begin{equation*}
    \KL(\pi \vert \rho \otimes \mu) = \int_{\X \times \Y} \left( \log \left( \frac{\dd \pi}{\dd \rho \otimes \mu}(x, y)\right) - 1 \right)\dd \pi(x, y)
\end{equation*}
if $\pi \ll \rho \otimes \mu$ and $+\infty$ otherwise.
When $\eps = 0$, problem \eqref{eq:quadratic-ot} corresponds to the usual quadratic optimal transport problem between $\rho$ and $\mu$, and the value of \eqref{eq:quadratic-ot} defines in this case the square of the $2$-Wasserstein distance $\Wass_2$ between $\rho$ and $\mu$. However, choosing $\eps > 0$ in \eqref{eq:quadratic-ot} has several advantages: first, it turns problem \eqref{eq:quadratic-ot} in a $\eps$-strongly-convex minimization problem, which enables the use of fast algorithms for its resolution \shortcite{Cuturi:2013, Sinkhorn_Rigollet, compOT_acc_grad, comp_OT, scaling_ot}. Second, problem \eqref{eq:quadratic-ot} enjoys better statistical properties when $\eps>0$ rather than when $\eps=0$, with improved sample complexity for its value \shortcite{sample_complexity_1, sample_complexity_2} and better guarantees when using stochastic optimization algorithms for its resolution \shortcite{genevay:hal-01321664, bercu:hal-02864967}. Thus, introducing for a solution $\pi^{(\mathrm{P}_\eps)}$ to \eqref{eq:quadratic-ot} the quantity
$$\Wass_{2,\eps}(\rho, \mu) = \left( \int_{\X \times \Y} \nr{x - y}^2 \dd \pi^{(\mathrm{P}_\eps)}(x,y) \right)^{1/2}, $$
one may hope that $\Wass_{2,\eps}$ approximates $\Wass_2$ well when $\eps$ is not too big. This fact has been the object of a long line of works, going to very recent developments. The convergence of $\Wass_{2, \eps}$ to $\Wass_2$ as $\eps$ goes to zero is established in general settings \shortcite{Mikami2004, leonard, bernton2021entropic, nutz2021entropic}, and it has been quantified in more specific settings. In the continuous setting, where both $\rho$ and $\mu$ are absolutely continuous, \shortcite{Adams2011, manh, erbar, Pal2019OnTD} gave first order asymptotics for $\Wass_{2, \eps}$ in terms of $\eps$ and thus showed in this setting an asymptotic linear rate of convergence of $\Wass_{2, \eps}$ to $\Wass_2$. These results were recently refined in \cite{CONFORTI2021108964} where second order asymptotics have been given. In the discrete setting, where both $\rho$ and $\mu$ are finitely supported, the rate of convergence of $\Wass_{2, \eps}$ to $\Wass_2$ was shown to be asymptotically exponential in \cite{Cominetti1994} in the context of the analysis of exponentially penalized finite dimensional linear programs. This result was then refined with a tight non-asymptotic analysis in \cite{weed-lp-exponential-cv}, enabling to choose $\eps$ in terms of the data in order to compute the unregularized cost $\Wass_2$ to a wanted precision.

Very different regimes are thus observed between the continuous setting (with a linear convergence rate) and the discrete setting (with an exponential convergence rate). However, very few was known -- until recently \cite{altschuler2021asymptotics} -- on the intermediate setting of semi-discrete optimal transport, where $\rho$ is absolutely continuous and $\mu$ is finitely supported, that is of particular importance in some applications. In statistics, it corresponds to the case where one wants to compare an empirical sample to a given probability measures, and it can serve to extend notions of quantiles and ranks to multivariate measures \shortcite{chernozhukov2017}. In numerical analysis, the semi-discrete setting gives a natural framework to approximate the solution of the optimal transport problem between a probability density $\rho$ and a probability measure $\mu$ that consists in approximating $\mu$ by a sequence of measures $(\mu_N)_{N\geq 1}$ with finite support such that $\lim_{N\to+\infty} \Wass_2(\mu,\mu_N) = 0$ \shortcite{oliker1989numerical, cullen1991generalised, gangbo1996geometry, caffarelli1999problem, mirebeau}. Finally in image processing, semi-discrete transport has proved useful for texture synthesis and style transfer \cite{texture1, texture2, texture3}. We thus focus in this work on the semi-discrete setting, and show that we can improve the recent asymptotic bounds given in \cite{altschuler2021asymptotics} under slightly stronger regularity assumptions on the source measure. In particular, we produce a non-asymptotic analysis of the dual solutions to problem \eqref{eq:quadratic-ot} in terms of $\eps$, which may be important in itself for the resolution of semi-discrete optimal transport using $\eps$-scaling techniques. 

The next section details the semi-discrete (regularized) optimal transport problem and gives our main results. Section \ref{sec:governing-ode} derives the ODE from which starts the proof of our main bound and the handling of its terms is done in Sections \ref{sec:strong_convexity} and \ref{sec:bound_2nd_term}. Section \ref{sec:numerical_illustration} finally illustrates our theoretical results on simple one-dimensional numerical examples.

\section{CONVERGENCE BOUNDS FOR SEMI-DISCRETE ENTROPIC OPTIMAL TRANSPORT}
\subsection{Semi-discrete Entropic Optimal Transport}
Let $\X$ be a compact subset of $\Rsp^d$ and $\rho \in \Prob(\X)$ be an absolutely continuous probability measure on $\X$. Let $\Y = \{y_1, \dots, y_N\} \subset \Rsp^d$ be a set of $N$ points in $\Rsp^d$ and let $\sigma$ be the counting measure associated to this set, i.e. $\sigma = \sum_{i=1}^N \delta_{y_i}$. Let $\mu = \sum_{i=1}^N \mu_i \delta_{y_i} \in \Prob(\Y)$ where for all $i, \mu_i \geq \underline{\mu} > 0$. Note that we will denote $R_\X, R_\Y > 0$ the smallest constants such that the $\X \subset B(0, R_\X), \Y \subset B(0, R_\Y)$ respectively, as well as $\diam(\X), \diam(\Y)$ the respective diameter of $\X, \Y$.

Back to problem \eqref{eq:quadratic-ot}, developing the square $\nr{x - y}^2$ and using that $\pi$ belongs to $\Pi(\rho, \mu)$, one can notice that this problem is equivalent to the following regularized maximum correlation problem:
\begin{equation}
    \label{eq:primal}
    \max_{\pi \in \Pi(\rho, \mu)} \int_{\X \times \Y} \sca{x}{y} \dd \pi(x,y) - \eps \KL(\pi \vert \rho \otimes \sigma ) \tag{$\mathrm{P}_\eps'$},
\end{equation}
with the relation $\left(\mathrm{P}_{2\eps}\right) = M_2(\rho) + M_2(\mu) -2\eps \mathcal{H}(\mu) - 2\times \eqref{eq:primal}$, where $M_2(\cdot)$ denotes the second moment of a probability measure and $\mathcal{H}(\cdot)$ its Shannon entropy.
By either $\eps$-strong concavity (when $\eps>0$) or Brenier's theorem \cite{Brenier} (when $\eps=0$, using that $\rho$ is absolutely continuous), problem \eqref{eq:primal} admits a unique solution denoted $\pi^\eps$. Moreover it admits the following (semi-)dual formulation and strong duality holds (see for instance Sections 2 of \cite{genevay:hal-01321664, bercu:hal-02864967}):
\begin{equation*}
    \min_{\psi \in \Rsp^N} \int_\X \psi^{c, \eps} \dd \rho + \sca{\psi}{\mu} + \eps,
\end{equation*}
where $\mu$ is conflated with the vector $(\mu_i)_{i=1,\dots,N}\in (\Rsp_+^*)^N$ and where $\psi^{c, \eps}$ corresponds to the \emph{$(c,\eps)$-transform} of $\psi$ when $\eps>0$ and to its \emph{Legendre transform} when $\eps=0$: $\forall x \in \X,$
$$ \psi^{c, \eps}(x) = \left\{
    \begin{array}{ll}
        \eps \log \left( \sum_{i=1}^N e^{\frac{\sca{x}{y_i} - \psi_i}{\eps}} \right)  & \mbox{if } \eps>0, \\
        \max_{i=1,\dots,N} \sca{x}{y_i} - \psi_i = \psi^*(x)  & \mbox{if } \eps=0.
    \end{array}
\right.
 $$
This dual problem is invariant to addition of constant vectors to $\psi$. We fix this invariance by adding the constraint that $\sca{\psi}{\mathbb{1}_N}= 0$ without any loss of generality, where $\mathbb{1}_N$ denotes a vector full of $1$'s in $\Rsp^N$. Introducing the (regularized) Kantorovich's functional $\Kant^\eps : \psi \mapsto \int_\X \psi^{c,\eps} \dd \rho + \eps$, one can then rewrite the dual formulation as
\begin{equation}
    \label{eq:dual}
    \min_{\psi \in \Rsp^N, \sca{\psi}{\mathbb{1}_N}=0} \Kant^\eps(\psi) + \sca{\psi}{\mu}. \tag{$\mathrm{D}_\eps$}
\end{equation}
The functional $\Kant^\eps$ is convex on $\Rsp^N$ and strictly convex on $\left(\mathbb{1}_N\right)^\top$: problem \eqref{eq:dual} admits a unique solution denoted $\psi^\eps$ (that we call later on a \emph{potential}), and such a solution verifies formally the following first-order condition: 
\begin{gather}
\label{eq:formal-first-order-cond}
    \nabla \Kant^\eps(\psi^\eps)=-\mu.
\end{gather}
More precisely, this first-order condition means that for all $i \in \{1, \dots, N\}$, if $\eps>0$,
\begin{gather}
\label{eq:first-order-cond}
    \mu(\{y_i\}) = \int_{x \in \X} e^{\frac{\sca{x}{y_i} - \psi^\eps_i - (\psi^\eps)^{c,\eps}(x)}{\eps}} \dd \rho(x) 
\end{gather}
and if $\eps=0$, $\mu(\{y_i\}) = \int_{x \in \X} \mathbb{1}_{\Lag_i(\psi^0)}(x) \dd \rho(x),$
where for any $\psi \in \Rsp^N$, $\Lag_i(\psi)$ denotes the $i$-th Laguerre cell w.r.t. $\psi$:
\begin{align*}
    \Lag_i(\psi) = \{x \in \X \vert \forall j, \sca{x}{y_i} - \psi_i \geq \sca{x}{y_j} - \psi_j \}.
\end{align*}
Note that the Laguerre cells are convex polytopes intersected with $\X$ and they define a tesselation of $\X$: $\bigcup_{i} \Lag_i(\psi^\eps) = \X$.

Finally, the primal-dual relationship that links the solution $\psi^\eps$ of problem \eqref{eq:dual} to the solution $\pi^\eps$ of problem \eqref{eq:primal} is the following: for all Borel set $A \subset \X$, all $i \in \{1, \dots, N\}$, if $\eps>0$,
\begin{gather}
    \label{eq:primal-dual}
    \pi^\eps(A, \{y_i\}) = \int_{x\in A} e^{\frac{\sca{x}{y_i} - \psi^\eps_i - (\psi^\eps)^{c,\eps}(x)}{\eps}} \dd \rho(x)
\end{gather}
and if $\eps=0$, $\pi^0(A, \{y_i\}) = \int_{x \in A} \mathbb{1}_{\Lag_i(\psi^0)}(x) \dd \rho(x)$.

\subsection{Non-asymptotic Behavior of Potentials}

The authors of \cite{altschuler2021asymptotics} recently tackled the question of the rate of convergence of $\Wass_{2, \eps}$ to $\Wass_2$ in the specific semi-discrete setting. They showed under regularity assumptions on the source $\rho$ the following asymptotic quadratic convergence in $\eps$ of the regularized cost (Theorem 1.1 in \cite{altschuler2021asymptotics}):
\begin{align*}
\label{eq:asymp-result-weed}
    \Wass_{2,\eps}^2(\rho, \mu) = \Wass_2^2(\rho, \mu) + \eps^2 \frac{\pi^2}{12}  \sum_{i<j}\frac{w_{ij}}{\nr{y_i - y_j}} + o(\eps^{2}),
\end{align*}
where $ w_{ij} = \int_{\Lag_i(\psi^0) \cap \Lag_j(\psi^0)} \rho(x) \dd \mathcal{H}^{d-1}(x).$
In order to show this, they demonstrated that the convergence of $\psi^\eps$ to $\psi^0$ as $\eps$ goes to $0$  happens at a rate faster that $\eps$ (Theorem 1.3 in \cite{altschuler2021asymptotics}):
$$ \lim_{\eps \to 0} \frac{1}{\eps}(\psi^\eps - \psi^0) = \dot{\psi^\eps}\Big\rvert_{\eps=0} =  0,$$
where $\dot{\psi^\eps} = \frac{\partial}{\partial \eps} \psi^\eps$. As discussed in \cite{altschuler2021asymptotics}, this result is unexpected because false in general optimal transport and stems from the particular setting of semi-discrete optimal transport with a positive source. In this work, we show that the result of Theorem 1.3 in \cite{altschuler2021asymptotics} can be extended and quantified to get a non-asymptotic control of $\dot{\psi^\eps}$, i.e. not only when $\eps \to 0$ but for $\eps \in \Rsp^*_+$. As in \cite{altschuler2021asymptotics}, we notice that regularity assumptions on the source measure $\rho$ are necessary to proceed with such controls. In particular, we make the following assumption (stronger than the one in \cite{altschuler2021asymptotics}):
\begin{assumption} \label{assump:source} The compact set $\X$ is convex. The source measure $\rho \in \Prob(\X)$ is absolutely continuous and its density (also denoted $\rho$), is bounded away from zero and infinity, i.e. there exist $m_\rho, M_\rho$ such that on $\X$, $$ 0 < m_\rho \leq \rho \leq M_\rho < +\infty.$$ 
\end{assumption}
\noindent Under this assumption and an Hölder continuity assumption on the density of $\rho$, we show the following behavior:
\begin{theorem}
\label{th:control-dot-psi}
Let $\rho \in \Prob(\X)$ satisfying Assumption \ref{assump:source} with an $\alpha$-Hölder continuous density for some $\alpha \in (0,1]$ and let $\mu \in \Prob(\Y)$. Then for any $\eps \leq 1$, $\alpha' \in (0, \alpha)$, the solutions $\psi^\eps$ to problem \eqref{eq:dual} verify:
\begin{align*}
    \nr{\dot{\psi^\eps}}_2 \lesssim \eps^{\alpha'}, 
\end{align*}
where $\dot{\psi^\eps} = \frac{\partial}{\partial \eps} \psi^\eps$ and $\lesssim$ hides multiplicative constants that depend on $\X, \rho, \Y, \mu$. Besides, for any $\eps \geq 1$,
\begin{align*}
    \nr{\dot{\psi^\eps}}_2 \lesssim 1 
\end{align*}
\end{theorem}
\begin{remark}[Constants]
A (very) rough upper bound on the hidden constants is given by the quantity
\small
\begin{align*}
    &\frac{N}{\underline{\mu}}\frac{M_\rho}{m_\rho}e^{R_\Y \diam(\X)}\Bigg( N R_\X \diam(\Y) + \log\frac{1}{\underline{\mu}} \\ 
    & + N^2M_\rho \diam(\X)^{d-1}(1 + \frac{C_\rho}{\delta^\alpha} + R_\X \diam(\Y) + \log\frac{1}{\underline{\mu}})\\
    &+ N^3 M_\rho \frac{\diam(\X)^{d-2}\diam(\Y)^4}{\cos(\theta /2)\delta^4}(1+R_\X \diam(\Y)+ \log\frac{1}{\underline{\mu}}) \Bigg),
\end{align*}
\normalsize
up to a multiplicative constant that depends only on the dimension. In this formula, $C_\rho$ is such that for any $x, x' \in \X$, $\abs{\rho(x) - \rho(x')} \leq C_\rho \nr{x - x'}^\alpha$, $\delta$ is the minimum distance between two points in $\Y$ and $\theta$ is the maximum angle formed by three not-aligned points in $\Y$. The dependence on $N$ is rather bad and it may be improved by replacing the $N^2$ term with $N$ times the maximum number of $(d-1)$-facets a Laguerre cell has in the tessellation $\bigcup_i \Lag_i(\psi^0)$ and the $N^3$ term with $N$ times the maximum number of $(d-2)$-facets a Laguerre cell has in this tessellation.
\end{remark}
\noindent This behavior is a consequence of the analysis of an ODE satisfied by the map $\eps \mapsto \psi^\eps$, and it is proven in Section \ref{sec:governing-ode}. An immediate consequence of this result concerns the quantitative stability of the mapping $\eps \mapsto \psi^\eps$. It also gives quantitative convergence results for $\psi^\eps$, $(\psi^\eps)^{c,\eps}$ and $\pi^\eps$ toward their different limits -- results that are reminiscent of the asymptotic ones of \cite{Cominetti1994} in the study of solutions of exponentially penalized finite dimensional linear programs.
\begin{corollary}
\label{coro:stab-potentials-eps}
Let $0 < \eps' \leq \eps$. Under assumptions of Theorem \ref{th:control-dot-psi}, denote $\psi^{\eps'}, \psi^{\eps}$ the solutions to problem \eqref{eq:dual} with regularization $\eps', \eps$ respectively. Then for any $\alpha' \in (0, \alpha)$,
\begin{align*}
    \nr{\psi^{\eps} - \psi^{\eps'}}_\infty \lesssim \eps^{\alpha'}(\eps - \eps').
\end{align*}
In particular, letting $\eps'$ go to $0$ yields
\begin{align*}
    \nr{\psi^\eps - \psi^0}_\infty &\lesssim \eps^{1+\alpha'}, \\
    \nr{(\psi^\eps)^{c,\eps} - (\psi^0)^*}_\infty &\lesssim \eps.
\end{align*}
Additionally, for $\rho$-a.e. $x \in \X$,
\begin{align*}
    \abs{(\psi^\eps)^{c,\eps}(x) - (\psi^0)^*(x)} &\lesssim \eps^{1 + \alpha'},\\
    \abs{\pi^\eps(x, \cdot) - \pi^0(x, \cdot)} &\lesssim e^{-c_x/\eps},
\end{align*}
where $c_x = \min_{ i \in \{1, \dots, N\}} \{  (\psi^0)^*(x) - \sca{x}{y_i} + \psi^0_i \quad \vert \quad \sca{x}{y_i} - \psi^0_i \neq (\psi^0)^*(x) \} > 0$.
\end{corollary}
\noindent This corollary is proven in Section \ref{sec:proof-cor-scaling} of the Appendix.
\begin{remark}[$\eps$-scaling]
Corollary \ref{coro:stab-potentials-eps} may be a first step, in the semi-discrete context, towards a mathematical justification of $\eps$-scaling or simulated annealing techniques used in the numerical resolution of optimal transport. Such techniques, reported for instance in \cite{KOSOWSKY1994477, scaling_ot, feydy:tel-02945979} in the context of Sinkhorn's algorithm for solving the assignment or discrete optimal transport problems using entropic regularization, are used to reduce the number of iterations necessary to compute a regularized solution. They consist in solving \eqref{eq:quadratic-ot} with a starting \emph{large} regularization parameter $\eps^0$, and then gradually decrease the regularization parameter over the course of the optimization, with a geometric decrease -- typically, $\eps^{k+1} = \eps^k/2$. The idea is that $\psi^{\eps^k}$ (or an approximation of it) is supposed to be a good starting point for an optimization algorithm that aims at estimating the solution $\psi^{\eps^{k+1}}$. This technique was introduced for Bertsekas’ auction algorithm \cite{Bertsekas1981, Bertsekas1988} for the resolution of the assignment problem, and it proved to reduce the worst case complexity from $O\left(\frac{N^2}{\eps}\right)$ to $O\left(N^3\log\left(\frac{1}{\eps}\right)\right)$ in order to get an $\eps$-approximate solution, where $N$ denotes the number of agents/tasks. Although successful in practice, similar reduction of the worst-case complexity of Sinkhorn's algorithm using the $\eps$-scaling strategy could not be proved, see the discussions in \cite{scaling_ot, feydy:tel-02945979} for more details.
\end{remark}
\begin{remark}[Exponential convergence of $\pi^\eps$]
The convergence of $\pi^\eps(x, y_i)$ to $\pi^0(x, y_i)$ at the rate $e^{-c_x/\eps}$ for $\rho$-a.e. $x$ and $i \in \{1, \dots, N\}$ matches in our semi-discrete setting the result of \cite{bernton2021entropic} that showed this rate of convergence, only asymptotically and for $(x, y_i)$ not in the support of $\pi^0$, but in a much more general setting.
\end{remark}

\subsection{Non-asymptotic Expansion of the Difference of Costs}

A consequence to these new bounds for $\eps \leq 1$ is an improvement of the asymptotic result on the convergence of the difference of costs proven in \cite{altschuler2021asymptotics}, to the following tight non-asymptotic result: 
\begin{theorem}
\label{th:suboptimality-cv}
Under assumptions of Theorem \ref{th:control-dot-psi}, for any $\alpha' \in (0, \alpha)$ and $\eps \leq 1$,
\begin{align*}
    \abs{ \Wass_{2,\eps}^2(\rho, \mu) - \Wass_2^2(\rho, \mu) - \eps^2 \frac{\pi^2}{12} \sum_{i<j}\frac{w_{ij}}{\nr{y_i - y_j}}} \lesssim \eps^{2+\alpha'}.
\end{align*}
\end{theorem}
\noindent This result and its tightness are respectively proved in Sections \ref{sec:asymptotics}, \ref{sec:tightness-th-suboptimality} of the Appendix.

\paragraph{Notation.}
The notation $\sca{v}{\nu}$ or $\Esp_\nu(v)$ is used to denote the quantity $\sum_{i=1}^N v_i \nu(y_i)$ when $v \in \Rsp^N$ and $\nu \in \Prob(\Y)$ or the quantity $\int_\X v \dd \nu$ when $v \in \L^1(\X)$ and $\nu \in \Prob(\X)$. Similarly, $\Var_\mu(v)$ denotes the quantity $\sum_{i=1}^N v_i^2 \mu(y_i) - \left(\sum_{i=1}^N v_i \mu(y_i)\right)^2$ in the first case and the quantity $\int_\X v^2 \dd \nu - \left( \int_\X v \dd \nu\right)^2$ in the second. For $v \in \Rsp^N$, $\diag(v)$ denotes the diagonal matrix of $\Rsp^{N\times N}$ with diagonal $v$.

\section{A GOVERNING ODE}

Similarly to \cite{Cominetti1994}, we show Theorem \ref{th:control-dot-psi} by leveraging the fact that $\eps \mapsto \psi^\eps$ satisfies a specific ODE that is deduced from the formal stationary equation \eqref{eq:formal-first-order-cond}. Let's fix $\eps > 0$ and recall in this case the expression of Kantorovich's functional $\Kant^\eps: \Rsp^N \to \Rsp$: for all $\psi \in \Rsp^N$,
\begin{equation}
    \label{eq:kant-discrete}
    \Kant^\eps(\psi) =  \int_\X \eps \log \left( \sum_{i=1}^N \exp\left( \frac{\sca{x}{y_i} - \psi_i}{\eps} \right) \right) \dd \rho(x) + \eps.
\end{equation}

Looking at equation \eqref{eq:kant-discrete}, one can establish that $\psi \mapsto \Kant^\eps(\psi)$ is a $\Class^2$ function from $\Rsp^N$ to $\Rsp$, with the following derivatives when evaluated in $\psi \in \Rsp^N$ (see e.g. Section 3 of \cite{bercu:hal-02864967}):
\begin{align}
    \label{eq:formulas-grad-kant}
    \nabla \Kant^\eps(\psi) &= - \Esp_{x \sim \rho} \pi^\eps_x(\psi), \\
    \label{eq:formulas-hessian-kant}
    \nabla^2 \Kant^\eps(\psi) &= \frac{1}{\eps} \Esp_{x \sim \rho} \left( \diag(\pi^\eps_x(\psi)) - \pi^\eps_x(\psi) \pi^\eps_x(\psi)^\top \right),
\end{align}
where for any $x \in \X$ and $\psi \in \Rsp^N$, $\pi^\eps_x(\psi)$ is a vector of $\Rsp^N$ whose components read for all $i \in \{1, \dots, N\}$
\begin{align}
    \label{eq:def-vector-pi}
    \pi^\eps_{x}(\psi)_i = 
    \frac{ \exp\left( \frac{ \sca{x}{y_i} - \psi_i}{\eps} \right) }{ \sum_{j=1}^N \exp\left( \frac{ \sca{x}{y_j} - \psi_j }{\eps} \right)   }. 
\end{align} 
Intuitively, one can interpret $x \mapsto \pi^\eps_{x}(\psi)_i$ as a smoothed version of the indicator function associated to the $i$-th Laguerre cell of $\psi$: it represents the ratio of mass sent from $x$ to $y_i$ proposed by the candidate solution $\psi$ to problem \eqref{eq:dual}.
One can also easily prove that for any $\psi \in \Rsp^N$, $\eps \mapsto \nabla \Kant^\eps(\psi)$ is a $\Class^1$ function from $\Rsp^*_+$ to $\Rsp^N$, with the formula
\begin{align}
    \label{eq:formulas-derivative-kant-eps}
    \frac{\partial}{\partial \eps}(\nabla \Kant^\eps)(\psi) = \frac{1}{\eps} \Esp_{x \sim \rho}  \big( &\diag(\pi^\eps_x(\psi))\log \pi^\eps_x(\psi) \\ \notag
     &- \pi^\eps_x(\psi) \pi^\eps_x(\psi)^\top \log \pi^\eps_x(\psi) \big).
\end{align}
We can then show:
\begin{proposition}
\label{prop:pties-potentials}
For any $\eps \geq 0$, denote $\psi^\eps$ the solution to problem \eqref{eq:dual}. The mapping $\eps \mapsto \psi^\eps$ is a $\Class^1$ function from $\Rsp_+^*$ to $\left(\mathbb{1}_N\right)^\perp$ that satisfies for any $\eps > 0$ the ODE
        \begin{equation}
            \label{eq:ode}
            \nabla^2 \Kant^\eps (\psi^\eps) \dot{\psi^\eps} + \frac{\partial}{\partial \eps}(\nabla \Kant^\eps)(\psi^\eps) = 0,
        \end{equation}
        where $\dot{\psi^\eps} = \frac{\partial}{\partial \eps} \psi^\eps. $
\end{proposition}
\begin{proof}
Since for any $\eps>0$, $\psi \mapsto \Kant^\eps(\psi)$ is a $\Class^2$ convex function from $\Rsp^N$ to $\Rsp$, one can characterize $\psi^\eps$ with the first order condition \eqref{eq:formal-first-order-cond}. Using that $\psi \mapsto \Kant^\eps(\psi)$ is $\Class^2$ on $\Rsp^N$, $\eps \mapsto \nabla \Kant^\eps(\psi)$ is $\Class^1$ on $\Rsp^+_*$, and that $\Kant^\eps$ is strictly convex on $\left(\mathbb{1}_N\right)^\perp$, i.e. that for any $\psi \in \Rsp^N$ such that $\sca{\psi}{\mathbb{1}_N} = 0$, $\nabla^2 \Kant^\eps(\psi)>0$, the implicit function theorem asserts that $\eps \mapsto \psi^\eps$ is a $\Class^1$ function from $\Rsp_+^*$ to $\left(\mathbb{1}_N\right)^\perp$. We can therefore differentiate the stationary equation \eqref{eq:formal-first-order-cond} w.r.t. $\eps$ and obtain that $\psi^\eps$ satisfies ODE \eqref{eq:ode}.
\end{proof}


Controlling $||\dot{\psi^\eps}||$ thus amounts to finding a lower bound on the p.s.d. matrix $\nabla^2 \Kant^\eps (\psi^\eps)$ and an upper bound on the second term $\frac{\partial}{\partial \eps}(\nabla \Kant^\eps)(\psi^\eps)$. 
The two following Theorems provide such controls. The first one can be regarded as a local strong-convexity estimate of the regularized Kantorovich's functional and is proven in Section \ref{sec:strong_convexity}:
\begin{theorem}[Strong convexity of $\Kant^\eps$]
\label{th:strong-convexity-K}
Let $\rho \in \Prob(\X)$ satisfying Assumption \ref{assump:source} and let $\mu \in \Prob(\Y)$. Then for any $\eps > 0$, the solution $\psi^\eps$ to problem \eqref{eq:dual} verifies for any $v \in \Rsp^N$
$$\Var_{\mu}(v) \leq \left( e^{R_\Y \diam(\X)} \frac{M_\rho}{m_\rho} + \eps \right) \sca{v}{\nabla^2 \Kant^\eps(\psi^\eps) v}.$$
\end{theorem}
\begin{remark}[Dependence on $\eps$]
\label{rk:strong-convexity-dep-eps}
Notice that in the limit $\eps \to 0$, one recovers a similar non-trivial estimate for the unregularized Kantorovich's functional given recently in Theorem 2.1 of \cite{delalande:hal-03164147} in the context of the study of the stability of optimal transport maps w.r.t. the target measure in quadratic optimal transport. Our estimate may also be compared to two other similar strong-convexity estimates found in Theorem 4 of \shortcite{luise} (in the discrete context) and in Lemma A.1 of \cite{bercu:hal-02864967} (that is not explicit) that both diverge as $\eps$ goes to zero. Note also that as $\eps$ goes to $\infty$, our estimate deteriorates: in this limit, $\Kant^\eps$ gets \emph{flat} around its minimum.
\end{remark}

\noindent The second control gives a uniform bound on the second term of ODE \eqref{eq:ode} and is proven in Section \ref{sec:bound_2nd_term}:
\begin{theorem}
\label{th:control-2nd-term}
Let $\rho \in \Prob(\X)$ satisfying Assumption \ref{assump:source} with an $\alpha$-Hölder continuous density for some $\alpha \in (0,1)$ and let $\mu \in \Prob(\Y)$. Then for any $\eps \leq 1$, $\alpha' \in (0, \alpha)$, the solutions $\psi^\eps$ to problem \eqref{eq:dual} verify:
\begin{align*}
    \nr{\frac{\partial}{\partial \eps}(\nabla \Kant^\eps)(\psi^\eps)}_\infty \lesssim \eps^{\alpha'}, 
\end{align*}
where $\lesssim$ hides multiplicative constants that depend on $\X, \rho, \Y, \mu$. Besides, for any $\eps \geq 1$,
\begin{align*}
    \nr{\frac{\partial}{\partial \eps}(\nabla \Kant^\eps)(\psi^\eps)}_\infty \lesssim \frac{1}{\eps}.
\end{align*}
\end{theorem} 
\noindent With these results, the proof of Theorem \ref{th:control-dot-psi} falls directly and is given in the Appendix (Section \ref{sec:proof-theorem-dot-psi}).


\label{sec:governing-ode}

\section{STRONG CONVEXITY OF $\Kant^\eps$}
\label{sec:strong_convexity}

In this section, we prove the strong convexity estimate of Theorem \ref{th:strong-convexity-K} for the regularized ($\eps > 0$) Kantorovich's functional. As mentionned in Remark \ref{rk:strong-convexity-dep-eps}, this estimate is reminiscent of Theorem 2.1 in \cite{delalande:hal-03164147} which relies on the Brascamp-Lieb inequality. In contrast, we use here the Prékopa-Leindler inequality \cite{prekopa1, leindler, prekopa2}, that is known to entail the Brascamp-Lieb inequality \cite{Bobkov2000}. This inequality allows us to get directly a similar estimate to the one given in Theorem \ref{th:strong-convexity-K}, but with a modified source measure:
\begin{proposition}
\label{prop:concavity-I}
The functional $I : \Rsp^N \to \Rsp, \psi \mapsto \log(\int_\X e^{-\psi^{c, \eps}})$ is $\Class^2$ and concave. In particular, its Hessian is negative semi-definite:
\begin{align*}
    \nabla^2 I(\psi^\eps) = &- \frac{1}{\eps} \Esp_{x \sim \tilde{\rho}^\eps} ( \diag(\pi^\eps_x(\psi^\eps)) - \pi^\eps_x(\psi^\eps) \pi^\eps_x(\psi^\eps)^\top ) \\
    &+ \Esp_{x \sim \tilde{\rho}^\eps}\pi^\eps_x(\psi^\eps) \pi^\eps_x(\psi^\eps)^\top \\
    &- \Esp_{x \sim \tilde{\rho}^\eps}\pi^\eps_x(\psi^\eps) \Esp_{x \sim \tilde{\rho}^\eps}\pi^\eps_x(\psi^\eps)^\top \leq 0,
\end{align*}
where $\tilde{\rho}^\eps := \frac{e^{-(\psi^\eps)^{c, \eps}}}{\int_\X e^{-(\psi^\eps)^{c, \eps}} }$.
\end{proposition}
This proposition implies Theorem \ref{th:strong-convexity-K}. Indeed, in the above expression of $\nabla^2 I(\psi^\eps)$, one can notice that the first term \emph{almost} corresponds to $- \nabla^2 \Kant(\psi^\eps)$, while the sum of the second and third terms \emph{almost} corresponds to a p.s.d. matrix whose associated bilinear form corresponds to the covariance w.r.t. $\mu = \Esp_{x \sim \rho} \pi^\eps_x(\psi^\eps)$. The difference with those terms resides in the presence of $\tilde{\rho}^\eps$ instead of $\rho$. A detailed proof of Theorem \ref{th:strong-convexity-K} is given in Section \ref{sec:proof-th-strong-convexity} of the Appendix.

\begin{proof}[Proof of Proposition \ref{prop:concavity-I}]
Let's first show that $I$ is $\Class^2$ on $\Rsp^N$. For any $x \in \X$ and $\psi \in \Rsp^N$, first recall the expression of $\psi^{c, \eps}(x)$:
\begin{align*}
    \psi^{c, \eps}(x) = \eps \log \left( \sum_{i=1}^N e^{\frac{\sca{x}{y_i} - \psi_i}{\eps}} \right)
\end{align*}
For any fixed $x \in \X$, the function $\psi \mapsto \psi^{c, \eps}(x)$ is $\Class^2$ on $\Rsp^N$ and its derivatives read
\begin{align*}
    \nabla_\psi [\psi^{c, \eps}(x)] &= - \pi^\eps_x(\psi), \\ 
    \nabla^2_\psi [\psi^{c, \eps}(x)] &= \frac{1}{\eps} \left( \diag(\pi^\eps_x(\psi)) - \pi^\eps_x(\psi) \pi^\eps_x(\psi)^\top \right).
\end{align*}
Therefore with $I : \psi \mapsto \log\left(\int_\X e^{-\psi^{c,\eps}}\right)$, $I$ is a $\Class^2$ function on $\Rsp^N$ and its derivatives read:
\begin{align*}
    \nabla I(\psi) &= \frac{- \int_\X\nabla_\psi [\psi^{c, \eps}(x)] e^{-\psi^{c,\eps}(x)} }{ \int e^{-\psi^{c,\eps}} } 
\end{align*}
and
\begin{align*}
    \nabla^2 I(\psi) &= \frac{- \int_\X\nabla^2_\psi [\psi^{c, \eps}(x)] e^{-\psi^{c,\eps}(x)} }{ \int e^{-\psi^{c,\eps}} } \\
    & \quad + \frac{\int_\X\nabla_\psi [\psi^{c, \eps}(x)]\nabla_\psi [\psi^{c, \eps}(x)]^\top e^{-\psi^{c,\eps}(x)}}{ \int e^{-\psi^{c,\eps}} } \\
    &\quad- \left( \frac{\int_\X\nabla_\psi [\psi^{c, \eps}(x)] e^{-\psi^{c,\eps}(x)} }{ \int e^{-\psi^{c,\eps}} } \right) \\
    &\quad \quad \times \left( \frac{\int_\X\nabla_\psi [\psi^{c, \eps}(x)] e^{-\psi^{c,\eps}(x)} }{ \int e^{-\psi^{c,\eps}} } \right)^\top,
\end{align*}
which entails the claimed expression for $\nabla^2 I(\psi^\eps)$.
We now show that $I$ is a concave function on $\Rsp^N$. Let $\psi, \varphi \in \Rsp^N$. Let $0<\lambda<1$. Notice that for any $u, v \in \X$ we have:
\begin{align*}
    \big( & \lambda  \psi  + (1-\lambda)\varphi \big)^{c, \eps}(\lambda u + (1-\lambda)v) \\
    &= \eps\log \left( \sum_{i=1}^N e^{\frac{\sca{\lambda u + (1-\lambda)v}{y_i} -  \left( \lambda \psi + (1-\lambda)\varphi \right)(y_i)}{\eps}}  \right) \\
    &= \eps\log \left( \sum_{i=1}^N \left( e^{\frac{\sca{u}{y_i} - \psi(y_i)}{\eps}} \right)^\lambda \left( e^{\frac{\sca{v}{y_i} - \varphi(y_i)}{\eps}} \right)^{1 - \lambda}  \right) \\
    &\leq \eps \log \left[ \left( \sum_{i=1}^N e^{\frac{\sca{u}{y_i} - \psi(y_i)}{\eps}} \right)^\lambda \left( \sum_{i=1}^N e^{\frac{\sca{v}{y_i} - \varphi(y_i)}{\eps}} \right)^{1 - \lambda}  \right] \\
    &= \lambda \psi^{c, \eps}(u) + (1-\lambda) \varphi^{c, \eps}(v),
\end{align*}
where the inequality corresponds to Hölder's inequality. Denoting
\begin{gather*}
    h(u) = e^{-\left( \lambda \psi + (1-\lambda)\varphi \right)^{c, \eps}(u)}, \\
    f(u) = e^{- \psi^{c, \eps}(u)}, \quad g(u) = e^{-\varphi^{c, \eps}(u)},
\end{gather*}
we thus have shown that 
$$ h(\lambda u + (1-\lambda)v) \geq f(u)^\lambda g(v)^{1-\lambda}. $$
Using that $\X$ is convex, the Prékopa–Leindler inequality \cite{prekopa1, leindler, prekopa2} then ensures that 
$$ \int_\X h \geq \left( \int_\X f \right)^\lambda \left( \int_\X g \right)^{1 - \lambda}.$$
This leads to the concavity of $I$:
\begin{align*}
    I(\lambda \psi + &(1-\lambda)\varphi) = \log\left(\int_\X h\right)\\
    &\geq \lambda \log \left( \int_\X f \right) + (1-\lambda) \log \left( \int_\X g \right) \\
    &= \lambda I(\psi) + (1-\lambda) I(\varphi). \qedhere
\end{align*}

\end{proof}

\section{BOUNDING $\frac{\partial}{\partial \eps}(\nabla \Kant^\eps)(\psi^\eps)$} 
\label{sec:bound_2nd_term}




In this section, we prove Theorem \ref{th:control-2nd-term} that gives a uniform bound on the second term $\frac{\partial}{\partial \eps}(\nabla \Kant^\eps)(\psi^\eps) \in \Rsp^N$ in ODE \eqref{eq:ode}. For conciseness, we will use $\pi^\eps_x$ instead of $\pi^\eps_x(\psi^\eps)$ since $\psi^\eps$ will be the only potential of interest. For any $j \in \{1, \dots, N\}$, we introduce the function
$$ f_j^\eps : x \in \X \mapsto \sca{x}{y_j} - \psi^\eps_j \in \Rsp.$$
Notice that for any $j \in \{1, \dots, N\}$, $\pi^\eps_{x,j}$ satisfies the following equality and inequality: $$\pi^\eps_{x, j} = \frac{ \exp(\frac{f^\eps_j(x)}{\eps}) }{ \sum_{k} \exp(\frac{f^\eps_k(x)}{\eps}) } \leq \exp\left(\frac{f^\eps_j(x) - \max_{\ell}f^\eps_\ell(x)}{\eps}\right).$$ 
Finding a uniform bound on the second term of \eqref{eq:ode} then consists in finding a bound for any $i \in \{1, \dots, N\}$ on the quantity
\begin{align}
    [\frac{\partial}{\partial \eps}(&\nabla \Kant^\eps)(\psi^\eps)]_i = \int_\X \frac{1}{\eps} [( \diag(\pi^\eps_x) - \pi^\eps_x (\pi^\eps_x)^\top ) \log \pi^\eps_x]_i \notag \\
    &= \int_\X \sum_{j \neq i} \left( \frac{f^\eps_i(x) - f^\eps_j(x)}{\eps^2} \right) \pi^\eps_{x,j} \pi^\eps_{x,i} \dd \rho(x).
    \label{eq:expr-Mpi-logpi}
\end{align}
Recall that the first order condition \eqref{eq:first-order-cond} entails that $e^{\frac{\psi^\eps_i}{\eps}} \mu_i = \int_\X e^{\frac{\sca{x}{y_i} - (\psi^\eps)^{c,\eps}(x)}{\eps}} \dd \rho(x)$, which ensures $\abs{\psi^\eps_i - \psi^\eps_j} \leq R_\X\abs{y_i - y_j} + \eps\abs{\log(\frac{\mu_i}{\mu_j})}$. Thus for any $x \in \X$ and $j \in \{1, \dots, N\}$,
\begin{align*}
    \abs{f_i^\eps(x) - f_j^\eps(x)} \leq 2 R_\X \diam(\Y) + \eps \abs{\log(\underline{\mu})}.
\end{align*}
Hence 
\begin{align*}
    \abs{[\frac{\partial}{\partial \eps}(\nabla \Kant^\eps)(\psi)]_i} \lesssim \frac{1}{\eps}.
\end{align*}
We now look for a more informative bound in the limit $\eps \to 0$. 
The quantity \eqref{eq:expr-Mpi-logpi} being an integral over $\X$, it will be in our interest to partition $\X$ into different subdomains where we can control the integrand. To this end, the Laguerre tessellation $\bigcup_i \Lag_i(\psi^\eps)$ already provides a first interesting partition. Recall the definition of the Laguerre cells with our new notation:
\begin{align*}
    \Lag_i(\psi^\eps) = \{x \in \X \vert \forall j, f_i^\eps(x) \geq f_j^\eps(x) \}.
\end{align*}
Figure \ref{fig:parition_X} gives an illustration of Laguerre cells, where the boundary of those cells are indicated by the plain black lines. In the control of \eqref{eq:expr-Mpi-logpi}, we will see that for $x$ in the \emph{interior} of $\Lag_i(\psi^\eps)$, $\pi^\eps_{x,j}$ is \emph{very small} for any $j \neq i$, and conversely for $x$ \emph{far} from $\Lag_i(\psi^\eps)$, $\pi^\eps_{x,i}$ is \emph{very small}. We will thus introduce two sets of points $\X_{i,\eta,+}^\eps, \X_{i,\eta,-}^\eps$ corresponding respectively to the points of $\Lag_i(\psi^\eps)$ and $\X \setminus \Lag_i(\psi^\eps)$ that are at a distance at least $\eta$ from the boundary of $\Lag_i(\psi^\eps)$. These sets are illustrated in Figure \ref{fig:parition_X} in green and blue respectively.
Now for $x$ close from the boundary of $\Lag_i(\psi^\eps)$ (i.e. $x \in \X \setminus (\X_{i,\eta,+}^\eps \cup \X_{i,\eta,-}^\eps)$, we will see that $\pi^\eps_{x, i}$ cannot be too small and there always exists a $j \neq i$ such that $\pi^\eps_{x, j}$ is not small neither. A finer treatment of those points close from the boundary of $\Lag_i(\psi^\eps)$ has to be carried out. For some $\gamma >0$, we will first define a set $A^\eps_{i, \eta,\gamma}$ of points that lie near the intersection between $\Lag_i(\psi^\eps)$ and another cell $\Lag_j(\psi^\eps)$, but that are at a distance at least $\sim \gamma$ from the other cells (i.e. $A^\eps_{i, \eta,\gamma}$ excludes the \emph{corners} of $\Lag_i(\psi^\eps)$) . This set is represented in yellow in Figure \ref{fig:parition_X}. On this set, only $\pi^\eps_{x, i}$ and $\pi^\eps_{x, j}$ are \emph{not small} and we show that their contributions to integral \eqref{eq:expr-Mpi-logpi} get compensated by symmetry w.r.t. the interface $\Lag_i(\psi^\eps) \cap \Lag_j(\psi^\eps)$.
Finally, we will denote the rest of $\X$ by $B_{i, \eta, \gamma}^\eps = \X \setminus (\X^\eps_{i,\eta,+} \cup \X^\eps_{i,\eta,-} \cup A_{i, \eta, \gamma}^\eps)$: this set corresponds to the areas in red in Figure \ref{fig:parition_X}. We will control \eqref{eq:expr-Mpi-logpi} on this set by leveraging two facts: its points are close from $\Lag_i(\psi^\eps)$ and its volume scales as $\sim \gamma^2$. 

\begin{figure}[h]
\centering
\includegraphics[scale=0.26]{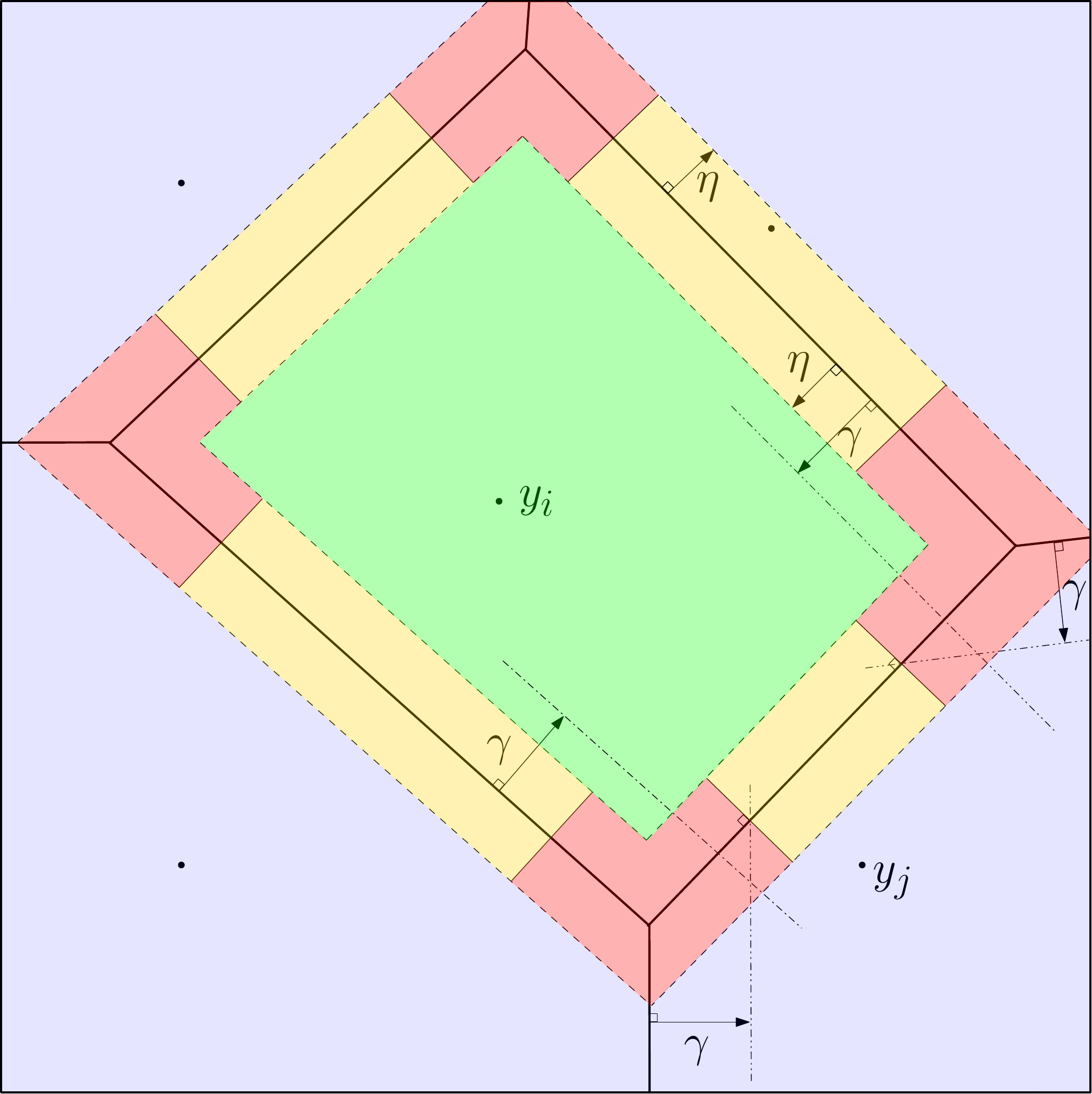}
\caption{Partition of $\X$: $\X^\eps_{i,\eta,+}$ is in green, $\X^\eps_{i,\eta,-}$ is in blue, $A_{i, \eta, \gamma}^\eps$ is in yellow and $B_{i, \eta, \gamma}^\eps$ is in red.}
\label{fig:parition_X}
\end{figure}
We make precise the definitions of the above mentioned sets in the proof of the following Proposition (Section \ref{sec:proof-prop-control-2nd-term-partition} of the Appendix), that allows to get the following bound:
\begin{proposition}
\label{prop:control-2nd-term-partition} Under assumptions of Theorem \ref{th:control-2nd-term}, for $0< \eps \leq 1, i \in \{1, \dots, N\}$ and for any $\eta, \gamma > 0$,
\begin{align*}
    [\frac{\partial}{\partial \eps}(\nabla \Kant^\eps)(\psi)]_i  &\lesssim \frac{1}{\eps^2} e^{-\eta /\eps} + \frac{\eta^{2+\alpha}}{\eps^2} + \frac{\gamma^2}{\eps^2} \left(\eta + e^{-\eta/\eps} \right) \\
    &+ \frac{1}{\eps^2}e^{-\tilde{\gamma}/\eps}\left(\eta + \eps\eta e^{\eta/\eps} - \eps^2(e^{\eta/\eps} - 1)\right),
\end{align*}
where $\tilde{\gamma} = \gamma \delta - \frac{\diam(\Y)^2}{\delta} \eta$ and $ \delta = \displaystyle \min_{i \neq j} \nr{y_i - y_j} > 0. $

\end{proposition}

The proof of the "$\eps \leq 1$" side of Theorem \ref{th:control-2nd-term} then follows from Proposition \ref{prop:control-2nd-term-partition} with an arbitrage on the quantities $\eta, \gamma$, see Section \ref{sec:proof-th-control-2nd-term} of the Appendix.

\section{NUMERICAL ILLUSTRATIONS}
\label{sec:numerical_illustration}

The code that generated the illustrations of this section is available at \href{https://github.com/alex-delalande/potentials-entropic-sd-ot}{https://github.com/alex-delalande/potentials-entropic-sd-ot}.

\subsection{Difference of Costs}
\label{sec:illustration-diff-costs}

Figure \ref{fig:suboptimalities} gives an illustration of Theorem \ref{th:suboptimality-cv} for a target $\mu = \frac{1}{2}(\delta_{-1} + \delta_1)$ and for four different source measures: 1. Lebesgue: $\rho(x) \propto \mathbb{1}_{[-1, 1]}(x)$; 2. Rescaled Gaussian: $\rho(x) \propto e^{-x^2/2\sigma^2} \mathbb{1}_{[-1, 1]}(x)$; 3. Rescaled Laplace: $\rho(x) \propto e^{-\abs{x}}\mathbb{1}_{[-1, 1]}(x)$; 4. $\frac{1}{2}$-Hölder density: $\rho(x) \propto (1 - \abs{x}^{1/2}) \mathbb{1}_{[-1, 1]}(x)$. For all these sources, we plot the absolute value of the difference of costs minus its asymptote as functions of $\eps$. The difference of costs is computed using the following formula given in Section 3 of \cite{altschuler2021asymptotics}:
$$ \Wass_{2, \eps}^2(\rho, \mu) - \Wass_{2}^2(\rho, \mu) = 8 \int_0^1 \frac{x}{1 + e^{4x/\eps}} \rho(x) \dd x.$$
Note that in these examples, $\eps \mapsto \psi^\eps$ is constant because of the symmetry of the problems. One can notice that for the cases of a Lebesgue or rescaled Gaussian source, the convergence of the difference of costs to its asymptote seems faster than the guaranteed $\eps^{3}$ of Theorem \ref{th:suboptimality-cv}. However, one can observe that Theorem \ref{th:suboptimality-cv} seems to give tight rates of convergence in the cases of a rescaled Laplace source or a $\frac{1}{2}$-Hölder source.

\begin{figure*}
\centering
\includegraphics[scale=0.23]{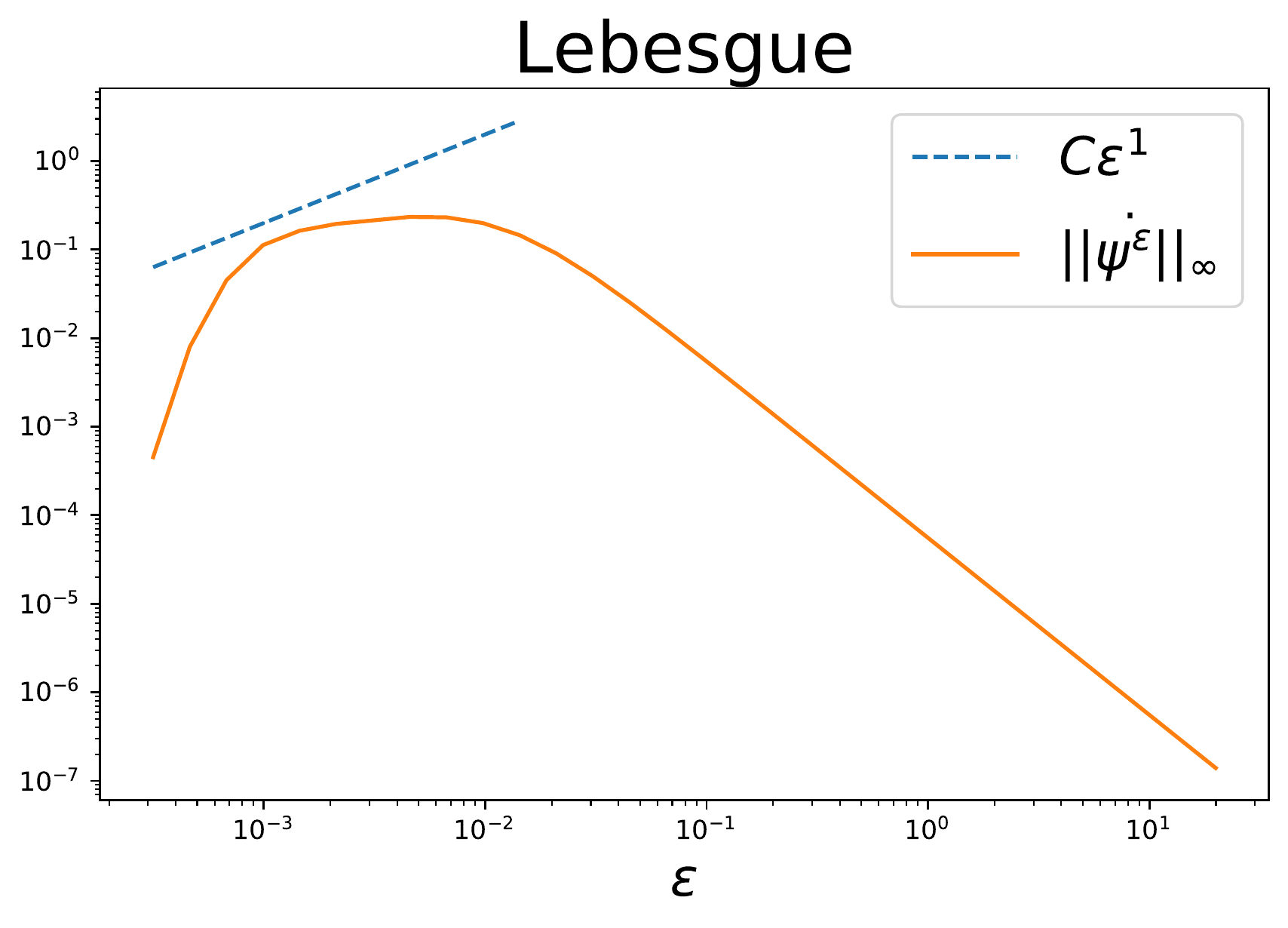}
\includegraphics[scale=0.23]{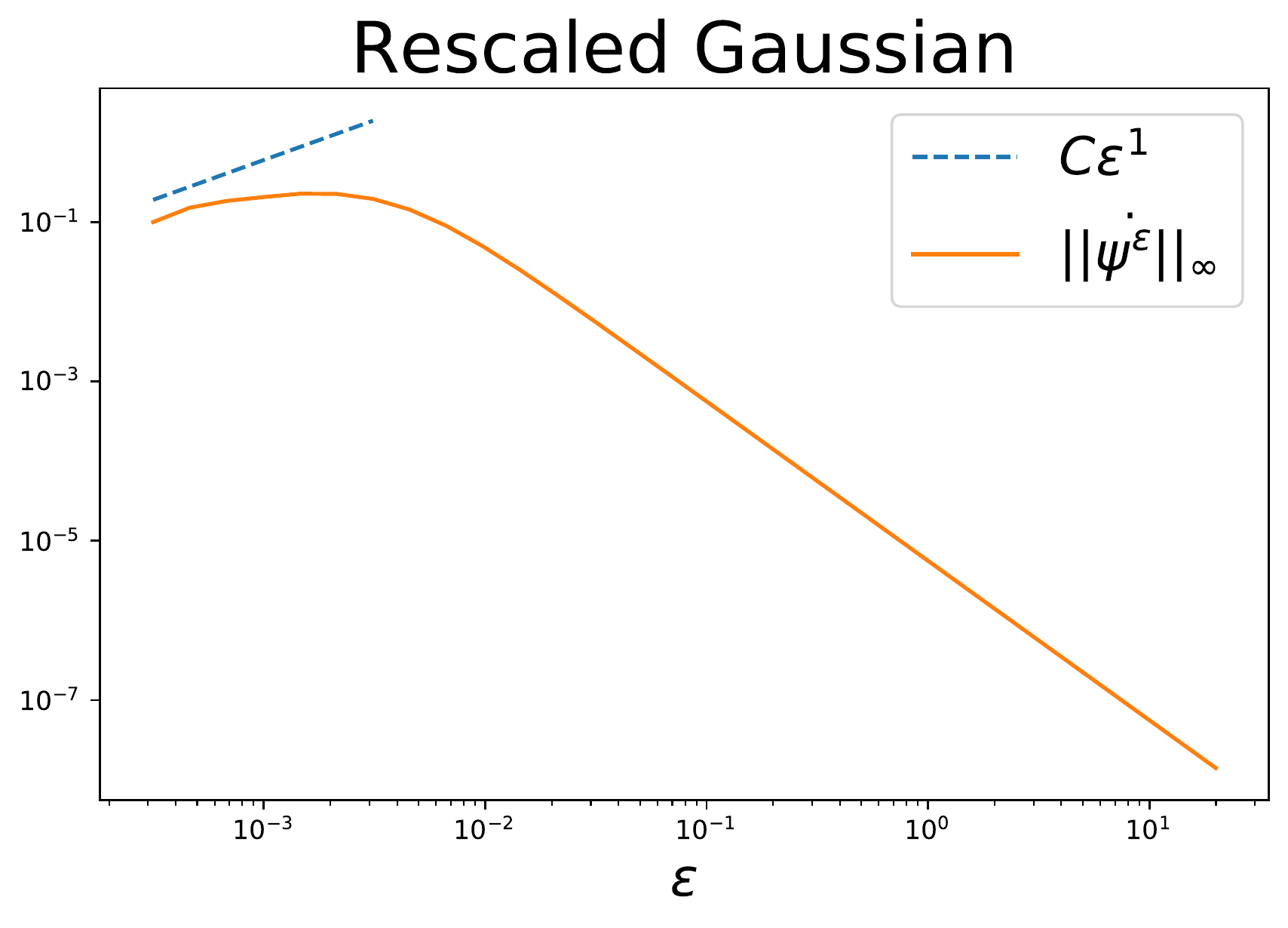}
\includegraphics[scale=0.23]{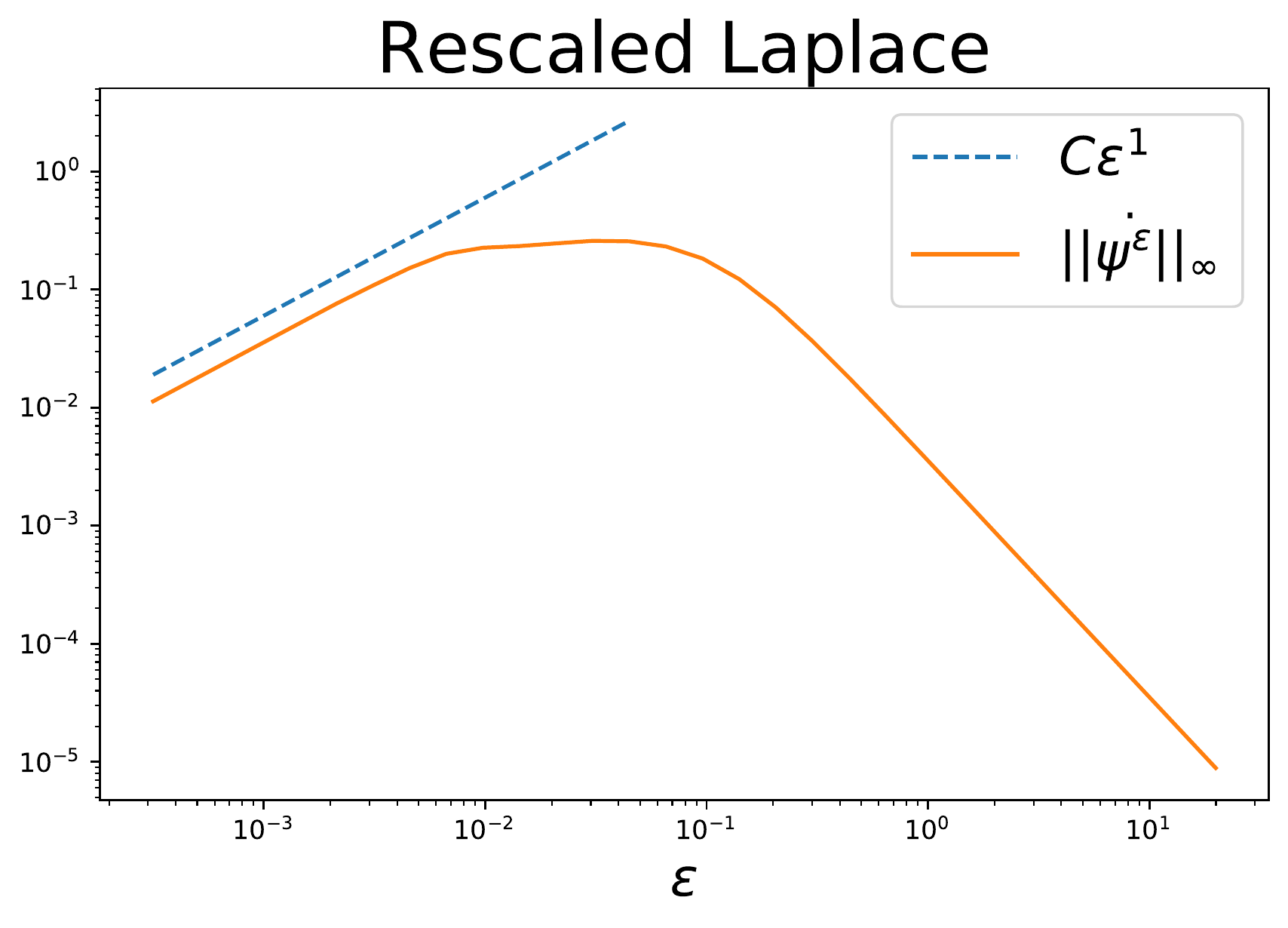}
\includegraphics[scale=0.23]{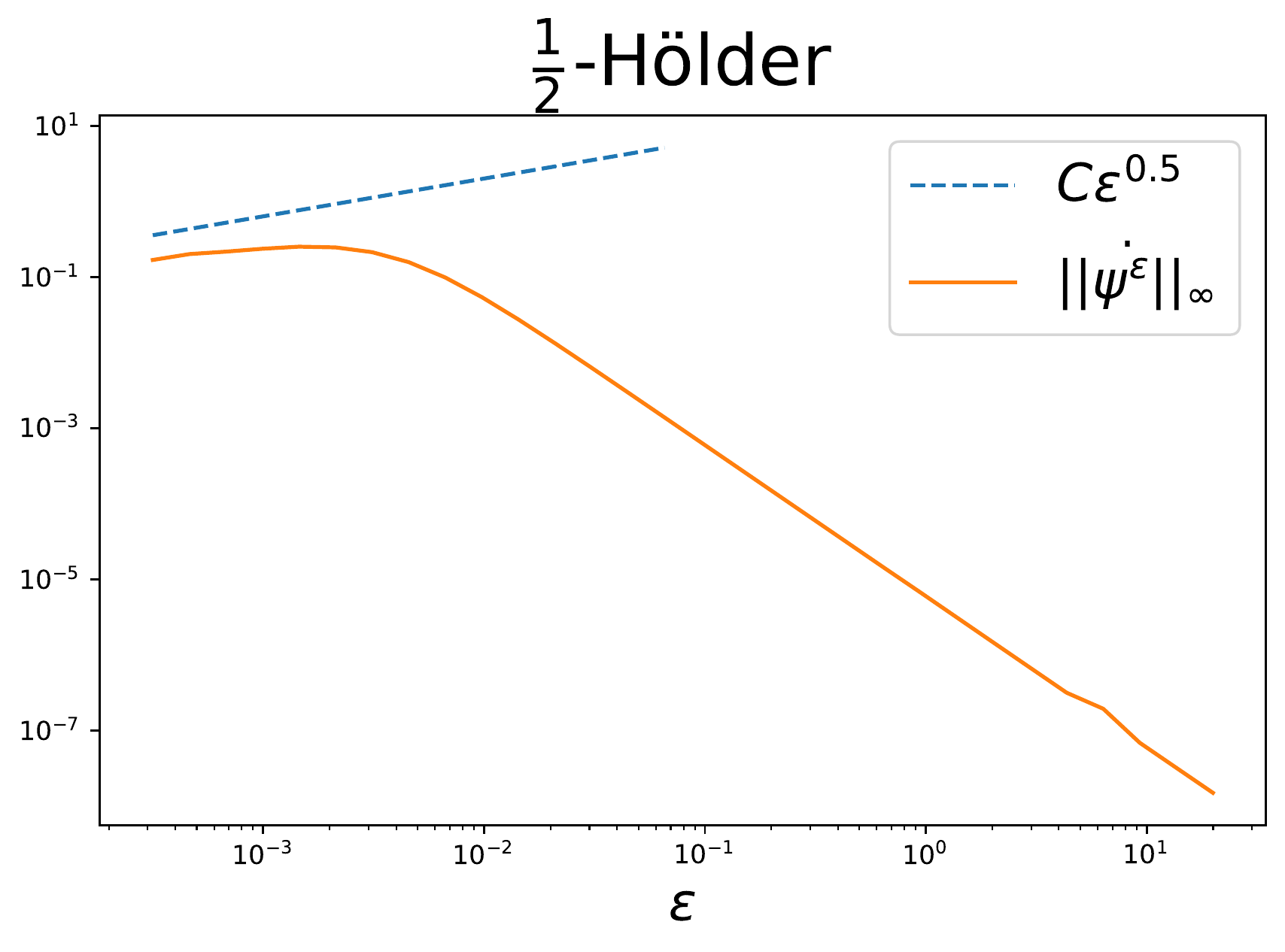}\\
\includegraphics[scale=0.23]{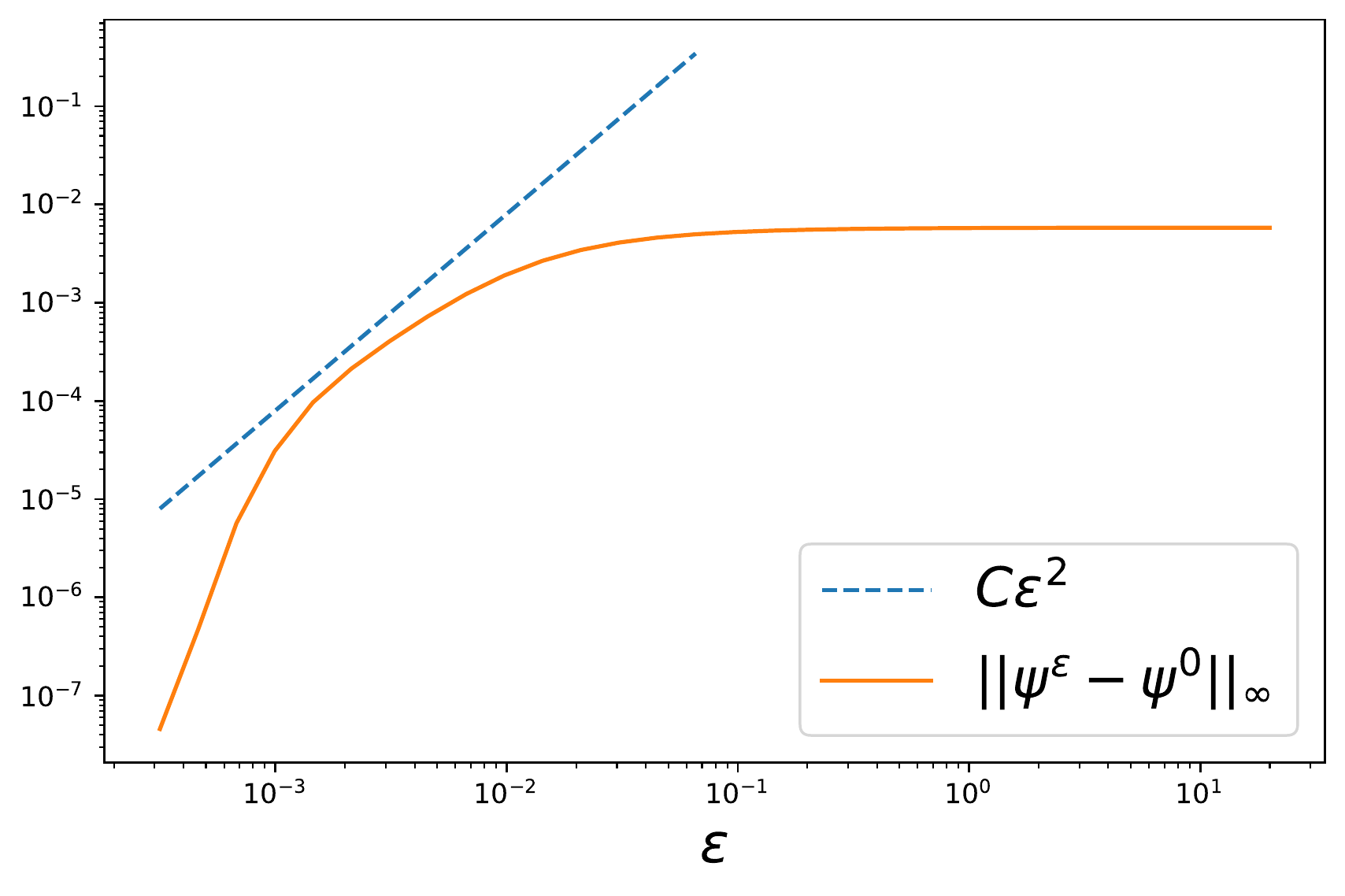}
\includegraphics[scale=0.23]{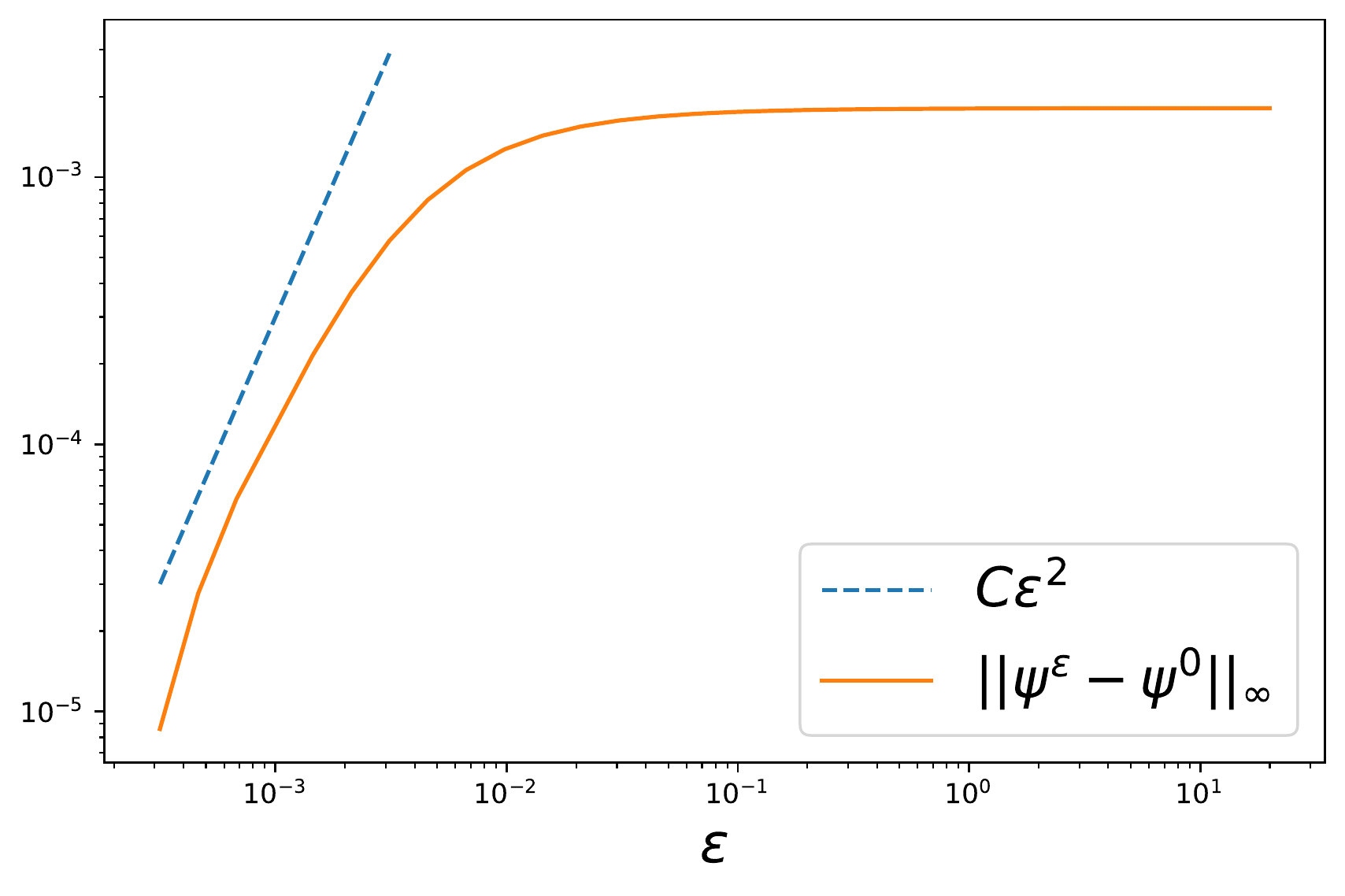}
\includegraphics[scale=0.23]{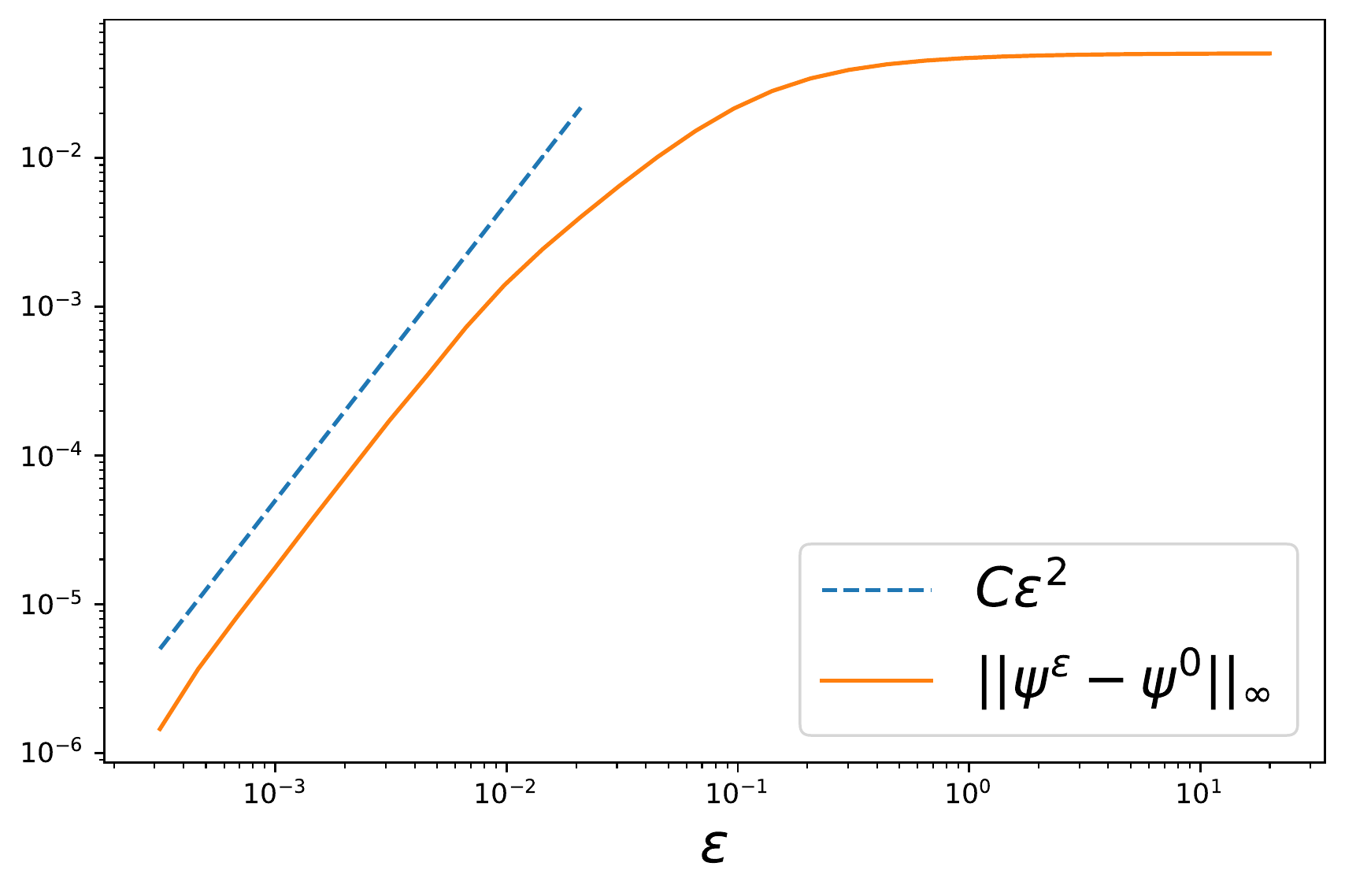}
\includegraphics[scale=0.23]{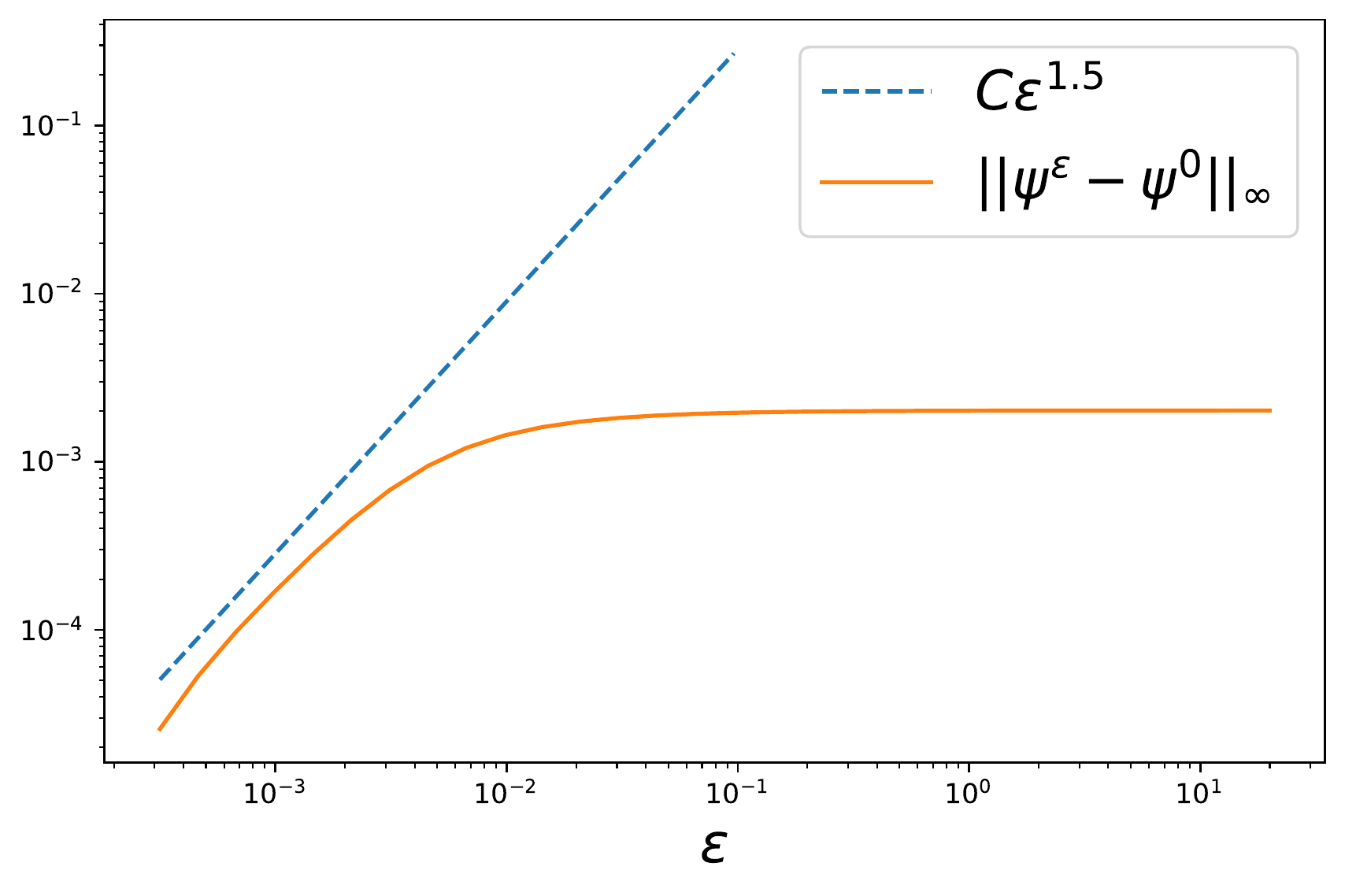}
\caption{(Top row) Behavior of $\eps \mapsto \nr{\dot{\psi^\eps}}_\infty$ for the 4 different sources and $\mu = \frac{1}{5} \sum_{i=1}^5 \delta_{y_i}$ with $(y_i)_{i=1, \dots, 5}$ randomly chosen. (Bottom row) Convergence of $\psi^\eps$ to $\psi^0$ for the four same examples.}
\label{fig:cv-psi-eps-psi-0}
\end{figure*}

\begin{figure}
\centering
\includegraphics[scale=0.23]{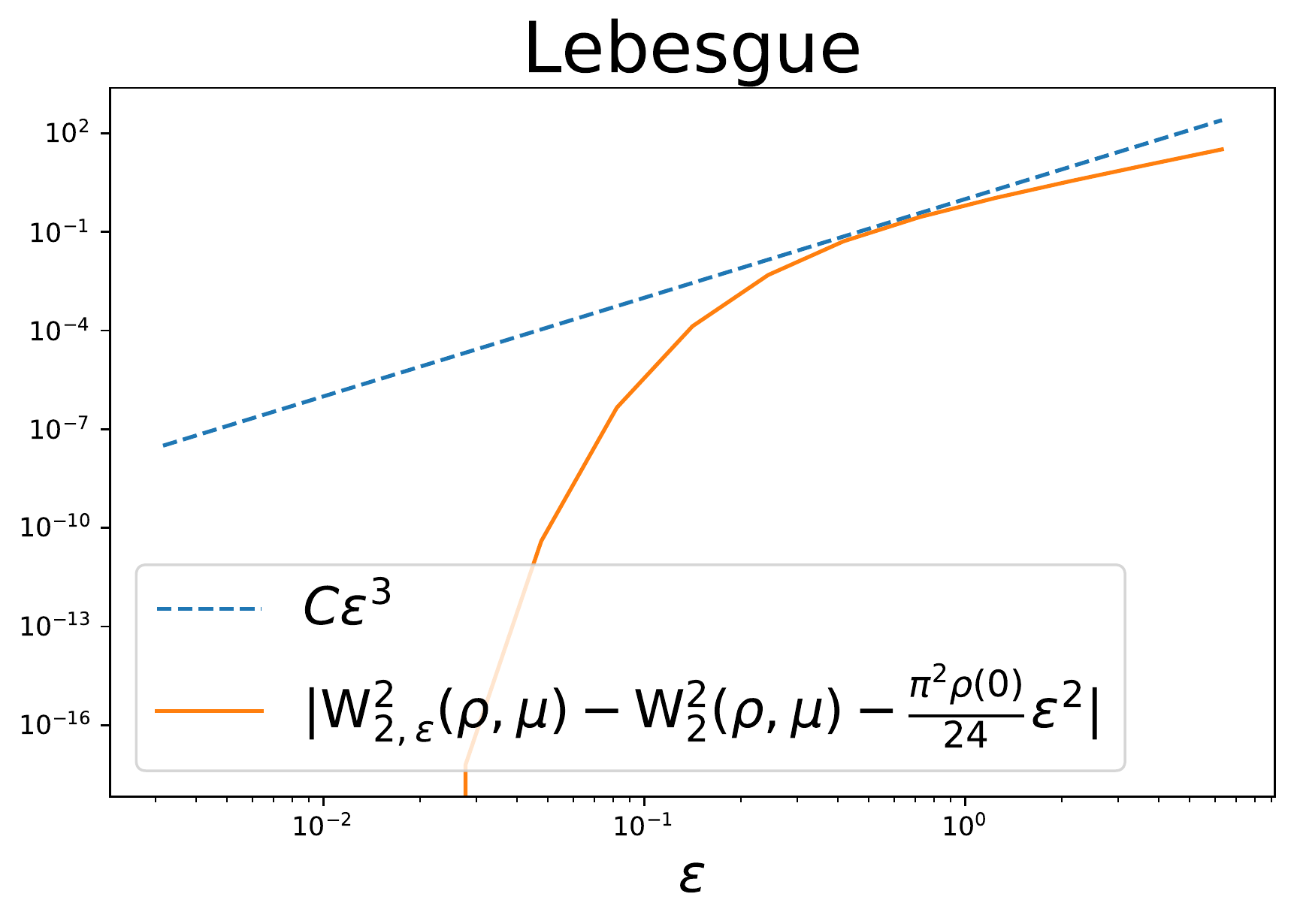}
\includegraphics[scale=0.23]{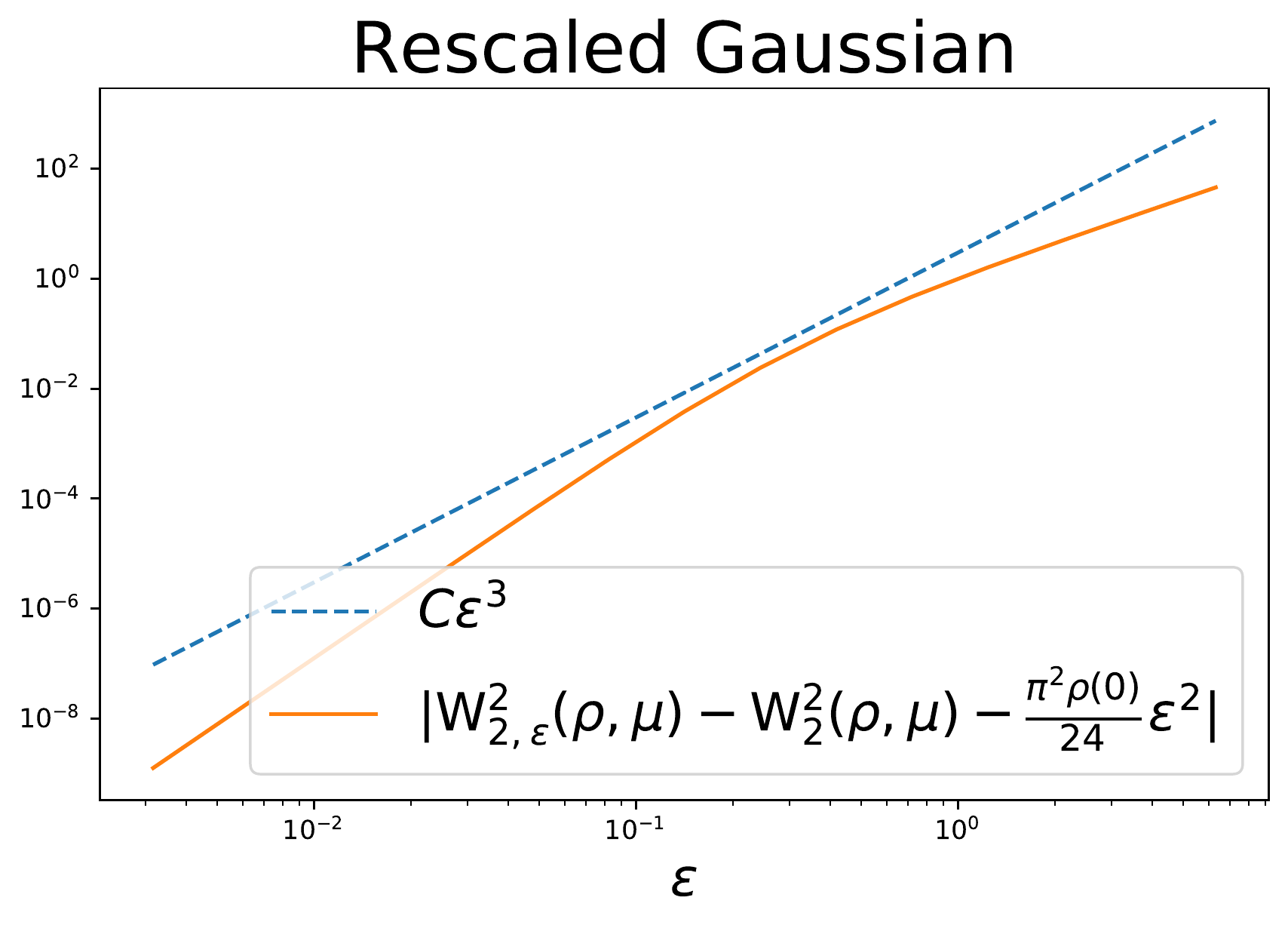}
\includegraphics[scale=0.23]{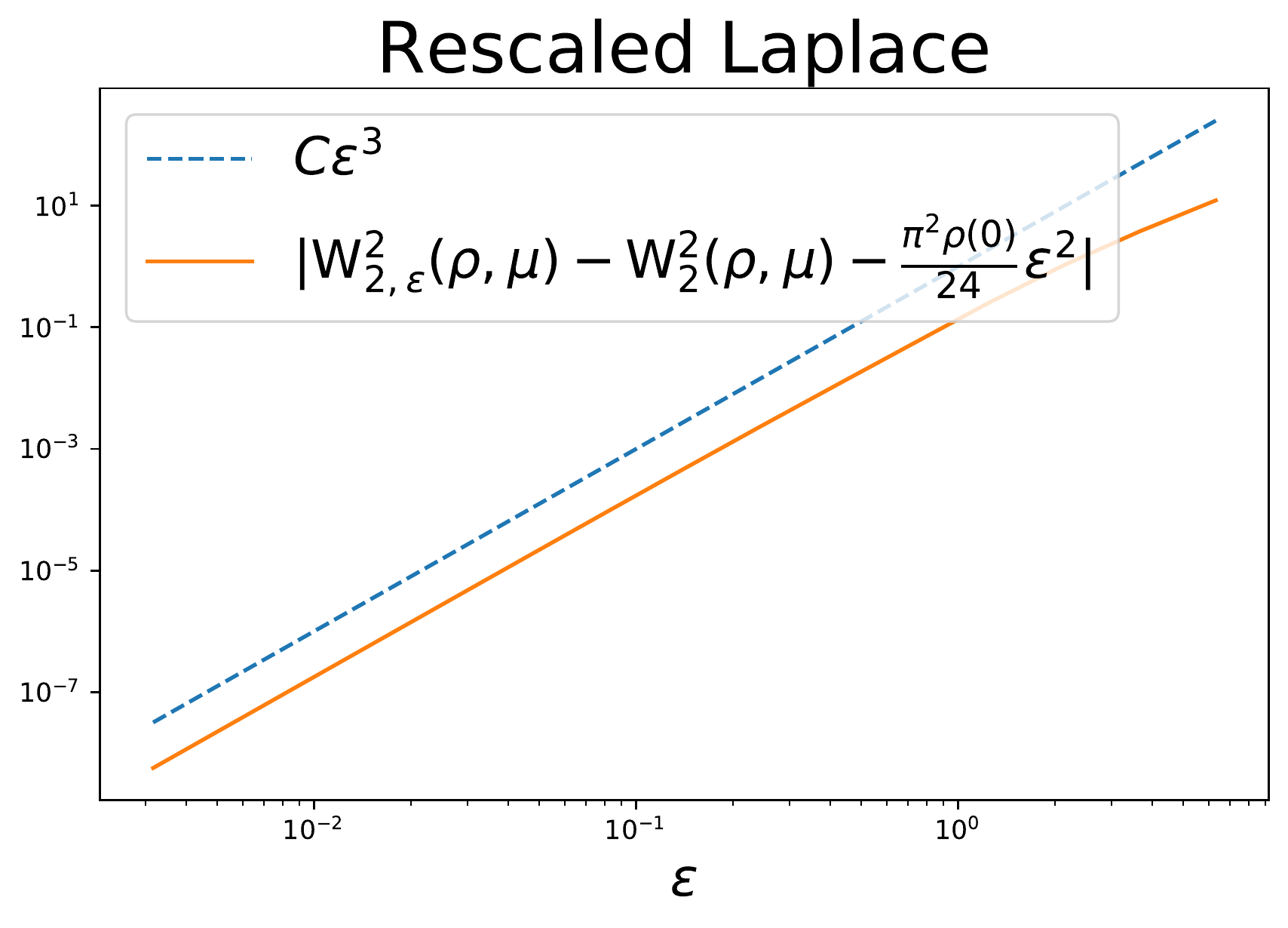}
\includegraphics[scale=0.23]{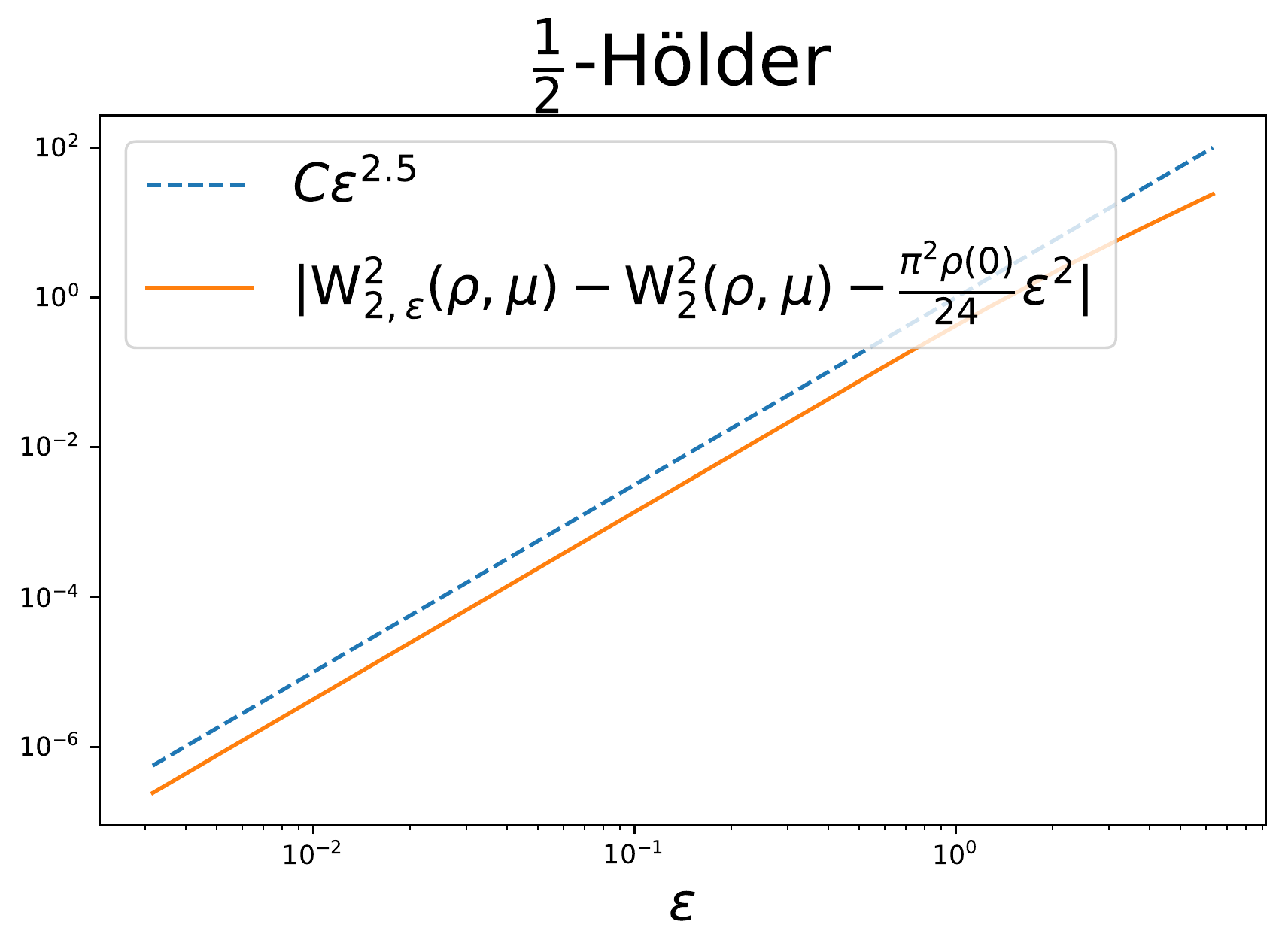}
\caption{Convergence of the difference of costs to its asymptote for the four different sources. The target is $\mu = \frac{1}{2}(\delta_{-1} + \delta_1)$.}
\label{fig:suboptimalities}
\end{figure}

\subsection{Behavior of $\eps \mapsto \psi^\eps$}

Figure \ref{fig:cv-psi-eps-psi-0} gives an illustration of Theorem \ref{th:control-dot-psi} and its Corollary \ref{coro:stab-potentials-eps}. We consider the same one dimensional sources as in the preceding section (up to restriction of their support to limit numerical errors).  We consider a target $\mu$ with 5 support points randomly chosen in the support of the source. We compute $\psi^\eps$ using the L-BFGS-B quasi-Newton method from SciPy \cite{SciPy}, where all integrals (appearing for instance in gradient computations) are also approximated using this package. Figure \ref{fig:cv-psi-eps-psi-0} represents in its first row the behavior of $\nr{\dot{\psi^\eps}}_\infty$ with respect to $\eps$ and compares the empirical results to the theoretical rates of Theorem \ref{th:control-dot-psi}. The derivative $\dot{\psi^\eps}$ is computed as the only solution in $(\mathbb{1}_N)^\perp$ to the linear system induced by the ODE \eqref{eq:ode}. Note that in Figure \ref{fig:cv-psi-eps-psi-0}, the long-time bound $\nr{\dot{\psi^\eps}}_\infty \lesssim 1$ for $\eps \geq 1$ seems to be loose, but this is specific to the setting where the target is uniform and this bound seems tight in general\footnote{Further experiments with a non-uniform target led to match empirically the long-time bounds of Theorem \ref{th:control-dot-psi}, see the GitHub repository.}.
The short-time case $\eps < 1$ however yields in the case of the Laplace and $\frac{1}{2}$-Hölder sources practical rates that seem to match the theoretical rate $\nr{\dot{\psi^\eps}}_\infty \lesssim \eps^\alpha$. The bottom row of Figure \ref{fig:cv-psi-eps-psi-0} gives an illustration of the convergence of $\psi^\eps$ to $\psi^0$. One can observe that the Lebesgue and Gaussian sources seem to enjoy faster rates of convergence than our theoretical rates. However, the Laplace and $\frac{1}{2}$-Hölder sources seem to yield potentials $\psi^\eps$ that converge to $\psi^0$ as fast as predicted in Corollary \ref{coro:stab-potentials-eps}.


\section{CONCLUSION}
We have given a non-asymptotic analysis of the solutions of entropic semi-discrete optimal transport in terms of the regularization parameter. We have shown that the dual solutions, sometimes called the Sinkhorn potentials, have a better than Lipschitz dependence on this regularization parameter. This may enable to derive faster algorithms for the numerical resolution of semi-discrete optimal transport based on $\eps$-scaling techniques and we leave this derivation for future work. Our analysis also entails tight and non-asymptotic bounds on the difference of costs, improving on the recent asymptotic expansion of \cite{altschuler2021asymptotics} and showing that this expansion does not admit in general a third order term.

\acknowledgments{The author warmly thanks Quentin Mérigot for his insight and his numerous comments. This work was done thanks to the support of the Agence Nationale de la Recherche through the project MAGA (ANR-16-CE40-0014).}

\bibliographystyle{apacite} 
\bibliography{ref}

\onecolumn
\appendix
\section{SUPPLEMENTARY MATERIAL}
\label{sec:appendix}

\subsection{Proof of Theorem \ref{th:control-dot-psi}}
\label{sec:proof-theorem-dot-psi}
\begin{proof}
For any $\eps > 0$, we can apply Proposition \ref{prop:pties-potentials} to $\psi^\eps$ and observe the relation 
\begin{align*}
    \nabla^2 \Kant^\eps (\psi^\eps) \dot{\psi^\eps} = - \frac{\partial}{\partial \eps}(\nabla \Kant^\eps)(\psi^\eps).
\end{align*}
Taking the scalar product of the last expression with $\dot{\psi^\eps}$ this gives:
\begin{equation*}
    \sca{\dot{\psi^\eps}}{\nabla^2 \Kant^\eps (\psi^\eps) \dot{\psi^\eps}} = - \sca{\dot{\psi^\eps}}{\frac{\partial}{\partial \eps}(\nabla \Kant^\eps)(\psi^\eps)}.
\end{equation*}
Applying Theorem \ref{th:strong-convexity-K} with $v = \dot{\psi^\eps}$ ensures that
\begin{align}
\label{eq:control-var-mu-dot-psi}
    \Var_{\mu}(\dot{\psi^\eps}) \leq -\left( e^{R_\Y \diam(\X)} \frac{M_\rho}{m_\rho} + \eps \right) \sca{\dot{\psi^\eps}}{\frac{\partial}{\partial \eps}(\nabla \Kant^\eps)(\psi^\eps)}.
\end{align}
Denote $\underline{\mu} > 0$ a positive real such that for all $i \in \{1, \dots, N\}, \mu(y_i) \geq \underline{\mu}$. Notice then that the facts that $\mu \geq \underline{\mu} \mathbb{1}_N$ and that $\sca{\dot{\psi^\eps}}{\mathbb{1}_N} = 0$ (because $\sca{\psi^\eps}{\mathbb{1}_N} = 0$) entail
\begin{align*}
    \Var_{\mu}(\dot{\psi^\eps}) &= \min_{m \in \Rsp} \nr{\dot{\psi^\eps} - m \mathbb{1}_N}^2_{\L^2(\mu)} \\
    &\geq \underline{\mu} \min_{m \in \Rsp} \nr{\dot{\psi^\eps} - m \mathbb{1}_N}^2_{2} = \underline{\mu} \nr{\dot{\psi^\eps}}^2_2.
\end{align*}
Using this inequality together with the Cauchy-Schwartz inequality in equation \eqref{eq:control-var-mu-dot-psi} we thus have
\begin{align*}
    \nr{\dot{\psi^\eps}}_2 \underline{\mu} &\leq \left( e^{R_\Y \diam(\X)} \frac{M_\rho}{m_\rho} + \eps \right) \nr{ \frac{\partial}{\partial \eps}(\nabla \Kant^\eps)(\psi^\eps) }_2 \notag \\
    &\leq N \left( e^{R_\Y \diam(\X)} \frac{M_\rho}{m_\rho} + \eps \right) \nr{ \frac{\partial}{\partial \eps}(\nabla \Kant^\eps)(\psi^\eps) }_\infty.
\end{align*}
Applying Theorem \ref{th:control-2nd-term} to the last inequality yields the wanted result.
\end{proof}

\subsection{Proof of Theorem \ref{th:strong-convexity-K}}
\label{sec:proof-th-strong-convexity}
\begin{proof}
Let $v \in \Rsp^N$. Notice that
\begin{align*}
    \sca{v}{\frac{1}{\eps} \Esp_{x \sim \tilde{\rho}^\eps} \left( \diag(\pi^\eps_x(\psi^\eps)) - \pi^\eps_x(\psi^\eps) \pi^\eps_x(\psi^\eps)^\top \right) v} = \int_\X \frac{1}{\eps} \Var_{\pi^{\eps}_{x}(\psi^\eps)}(v) \dd \tilde{\rho}^\eps(x),\\
    \sca{v}{\Esp_{x \sim \tilde{\rho}^\eps}\pi^\eps_x(\psi^\eps) \pi^\eps_x(\psi^\eps)^\top v} - \sca{v}{\Esp_{x \sim \tilde{\rho}^\eps}\pi^\eps_x(\psi^\eps) \Esp_{x \sim \tilde{\rho}^\eps}\pi^\eps_x(\psi^\eps)^\top v} = \Var_{x \sim \tilde{\rho}^\eps}( \Esp_{\pi^{\eps}_{x}(\psi^\eps)}(v) ).
\end{align*}
Thus Proposition \ref{prop:concavity-I} ensures that
\begin{equation}
    \label{eq:brascamp-lieb}
    \Var_{x \sim \tilde{\rho}^\eps}( \Esp_{\pi^{\eps}_{x}(\psi^\eps)}(v) ) \leq  \int_\X \frac{1}{\eps} \Var_{\pi^{\eps}_{x}(\psi^\eps)}(v) \dd \tilde{\rho}^\eps(x),
\end{equation}
where we recall $$ \tilde{\rho}^\eps = \frac{e^{-(\psi^\eps)^{c, \eps}}}{\int_\X e^{-(\psi^\eps)^{c, \eps}} } = \frac{e^{-(\psi^\eps)^{c, \eps}}}{Z}. $$ From the definition of the $(c,\eps)$-transform, one can see that $(\psi^\eps)^{c,\eps}$ is $R_\Y$-Lipschitz (see Proposition 17, Chapter 3 in \cite{genevay:tel-02319318}). This ensures that $(\psi^\eps)^{c,\eps}$ is bounded on the compact set $\X$: on this set, there exists constants $m, M \in \Rsp$ such that
$$ m \leq (\psi^\eps)^{c,\eps} \leq M \quad \text{and} \quad M - m \leq R_\Y \diam(\X).$$
This gives the control
$$ \frac{e^{-M}}{Z} \leq \tilde{\rho}^\eps \leq \frac{e^{-m}}{Z}. $$
Recalling that $m_\rho \leq \rho \leq M_\rho$ we thus have:
$$ \frac{e^{-M}}{Z M_\rho} \rho \leq \tilde{\rho}^\eps \leq \frac{e^{-m}}{Z m_\rho} \rho. $$
This control, \eqref{eq:formulas-hessian-kant} and \eqref{eq:brascamp-lieb} thus give
\begin{align}
\label{eq:control-2nd-derivative}
    \Var_{x \sim \rho}( \Esp_{\pi^{\eps}_{x}(\psi^\eps)}(v) ) &\leq e^{R_\Y \diam(\X)} \frac{M_\rho}{m_\rho}  \int_\X \frac{1}{\eps} \Var_{\pi^{\eps}_{x}(\psi^\eps)}(v) \dd \rho(x) \notag \\
    &= e^{R_\Y \diam(\X)} \frac{M_\rho}{m_\rho} \sca{v}{\nabla^2 \Kant^\eps(\psi^\eps) v}.
\end{align}
Recall that from the first order condition \eqref{eq:first-order-cond} and expression \eqref{eq:formulas-grad-kant}, we get
$$ \mu = \Esp_{x \sim \rho} \pi^\eps_x(\psi^\eps). $$
Hence using the associativity of variances we have:
$$ \Var_{\mu}(v) = \Var_{x \sim \rho}( \Esp_{\pi^{\eps}_{x}(\psi^\eps)}(v) ) + \int_\X \Var_{\pi^{\eps}_{x}(\psi^\eps)}(v) \dd \rho(x), $$
so that using again \eqref{eq:formulas-hessian-kant},
$$ \Var_{x \sim \rho}( \Esp_{\pi^{\eps}_{x}(\psi^\eps)}(v) )  = \Var_{\mu}(v) - \eps \sca{v}{\nabla^2 \Kant^\eps(\psi^\eps) v}. $$
Injecting this last equality into \eqref{eq:control-2nd-derivative} yields the desired result.
\end{proof}

\subsection{Proof of Proposition \ref{prop:control-2nd-term-partition}}
\label{sec:proof-prop-control-2nd-term-partition}

\begin{proof}
We introduce for any $i \in \{1, \dots, N\}$ and parameter $\eta>0$ the sets:
\begin{align*}
    \X^\eps_{i,\eta,+} &= \{x \in \Lag_i(\psi^\eps) \vert \forall j \neq i, \frac{f_i^\eps(x) - f_j^\eps(x)}{\nr{y_i - y_j}} \geq \eta \},\\
    \X^\eps_{i,\eta,-} &= \{x \in \X \vert \forall j \in \arg \max_{\ell} f_\ell^\eps(x), \frac{f_j^\eps(x) - f_i^\eps(x)}{\nr{y_j - y_i}} \geq \eta \},
\end{align*}
that correspond respectively to the points of $\Lag_i(\psi^\eps), \X\setminus\Lag_i(\psi^\eps)$ that are at a distance at least $\eta$ from the boundary of $\Lag_i(\psi^\eps)$ and that are illustrated in green and in blue in Figure \ref{fig:parition_X}. 
We then define for any $j \neq i$ the common boundary between $\Lag_i(\psi^\eps)$ and $\Lag_j(\psi^\eps)$:
$$ H_{ij} = \Lag_i(\psi^\eps) \cap \Lag_j(\psi^\eps).$$ Next, for a parameter $\gamma > 0$, define the set of points of $H_{ij}$ that are at a distance at least $\gamma$ from the other Laguerre cells:
\begin{align*}
    H_{ij}^{-\gamma} = \{x^0 &\in H_{ij} \vert \forall k\neq i,j, f_i^\eps(x^0) = f_j^\eps(x^0) \geq f_k^\eps(x^0) + \gamma \max(\nr{y_i - y_k}, \nr{y_j - y_k})\}.
\end{align*}
Then define
\begin{align*}
    A_{i, \eta, \gamma}^\eps = \bigcup_{j \neq i} \{ &x^0 + t d_{ij}, x^0 \in H_{ij}^{-\gamma}, t \in [-\eta\nr{y_i - y_j}, +\eta\nr{y_i - y_j}] \}
\end{align*}
where $d_{ij} = \frac{y_i - y_j}{\nr{y_i - y_j}^2}$. This set corresponds to the areas in yellow in Figure \ref{fig:parition_X}.
Define finally $B_{i, \eta, \gamma}^\eps = \X \setminus (\X^\eps_{i,\eta,+} \cup \X^\eps_{i,\eta,-} \cup A_{i, \eta, \gamma}^\eps)$: this set corresponds to the areas in red in Figure \ref{fig:parition_X}.

\textbf{Control on $\X^\eps_{i,\eta,+}$.}
For $x \in \X^\eps_{i,\eta,+}$, we have for any $j \neq i$,
$$ \pi_{x,j}^\eps \leq \exp\left( \frac{f_j^\eps(x) - f_i^\eps(x)}{\eps} \right) \leq e^{-\eta \nr{y_i - y_j}/\eps} \leq e^{-\eta \delta /\eps}.$$
This gives in equation \eqref{eq:expr-Mpi-logpi} the control 
\begin{equation}
    \label{eq:control-partition-1}
    \forall x \in \X^\eps_{i,\eta,+}, \quad \sum_{j=1, j \neq i}^N \left( \frac{f^\eps_i(x) - f^\eps_j(x)}{\eps^2} \right) \pi^\eps_{x,j} \pi^\eps_{x,i} \lesssim \frac{1}{\eps^2} e^{-\eta \delta/\eps} \lesssim \frac{1}{\eps^2} e^{-\eta/\eps}.
\end{equation}
\textbf{Control on $\X^\eps_{i,\eta,-}$.}
For $x \in \X^\eps_{i,\eta,-}$, we have 
$$ \pi_{x,i}^\eps \leq \exp\left( \frac{f_i^\eps(x) - \max_j f_j^\eps(x)}{\eps} \right) \leq e^{-\eta\delta/\eps}.$$
This gives in equation \eqref{eq:expr-Mpi-logpi} the control 
\begin{align}
    \label{eq:control-partition-2}
    \forall x \in \X^\eps_{i,\eta,+}, \quad \sum_{j=1, j \neq i}^N \left( \frac{f^\eps_i(x) - f^\eps_j(x)}{\eps^2} \right) \pi^\eps_{x,j} \pi^\eps_{x,i} \lesssim \frac{1}{\eps^2} e^{-\eta/\eps}.
\end{align}

\textbf{Control on $A_{i, \eta, \gamma}^\eps$.}
For any $x \in A_{i, \eta, \gamma}^\eps$, there exists $j \in \{1, \dots, N\}$, $x^0 \in H_{ij}^{-\gamma}$ and $t \in [-\eta\nr{y_i - y_j}, +\eta\nr{y_i - y_j}]$ such that
$$ x = x^0 + t d_{ij}.$$
For such a point, we have
\begin{align}
    f_i^\eps(x) - f_j^\eps(x) &= \sca{x^0 + t d_{ij}}{y_i - y_j} - \psi^\eps_i + \psi^\eps_j \notag \\
    &= f_i^\eps(x^0) - f_j^\eps(x^0) + t \sca{d_{ij}}{y_i - y_j} \notag \\
    &= t.
    \label{eq:diff-1}
\end{align}
Moreover, for any $k \neq i,j$ we have by definition of $H_{ij}^{-\gamma}$
\begin{align}
    f_i^\eps(x) - f_k^\eps(x) &= f_i^\eps(x^0) - f_k^\eps(x^0) + t \sca{d_{ij}}{y_i - y_k} \notag \\
    &\geq \gamma \nr{y_i-y_k} - \abs{t} \frac{\diam(\Y)}{\delta} \notag \\
    &\geq \gamma \delta - \frac{\diam(\Y)^2}{\delta} \eta := \tilde{\gamma}.
    \label{eq:diff-2}
\end{align}
In the same way,
\begin{align}
    f_j^\eps(x) - f_k^\eps(x) \geq \gamma \delta - \frac{\diam(\Y)^2}{\delta} \eta = \tilde{\gamma}.
    \label{eq:diff-3}
\end{align}
The integral we want to control on $A_{i, \eta, \gamma}^\eps$ reads:
\begin{align}
\label{eq:integral-A}
    \sum_{j} \int_{x^0 \in H_{ij}^{-\gamma}} \int_{t=0}^{\eta \nr{y_i - y_j}} (g_i^\eps(x^0 - td_{ij}) + g_i^\eps(x^0 + td_{ij})) \dd t \dd \mathcal{H}^{d-1}(x^0)
\end{align}
where $g_i^\eps(x) = \sum_{j=1, j \neq i}^N \left( \frac{f^\eps_i(x) - f^\eps_j(x)}{\eps^2} \right) \pi^\eps_{x,j} \pi^\eps_{x,i} \rho (x)$.

Let's find an upper bound on $\abs{g_i^\eps(x^0 - td_{ij}) + g_i^\eps(x^0 + td_{ij})}$. To simplify the notation, denote $x^- = x^0 - t d_{ij}$ and $x^+ = x^0 + t d_{ij}$. Using equation \eqref{eq:diff-1}, we have the expression
\begin{align}
    g_i^\eps(x^-) &= \left( \frac{f_i^\eps(x^-) - f_j(x^-)}{\eps^2}\right)  \pi^\eps_{x^-, j} \pi^\eps_{x^-, i} \rho(x^-) + \sum_{k \neq i,j} \left( \frac{f_i^\eps(x^-) - f_k(x^-)}{\eps^2}\right)  \pi^\eps_{x^-, k} \pi^\eps_{x^-, i}\rho(x^-) \notag \\
    &= \frac{-t}{\eps^2} \pi^\eps_{x^-, j} \pi^\eps_{x^-, i}\rho(x^-) + S(x^-),
    \label{eq:expr-M-x-}
\end{align}
where $S(x) = \sum_{k \neq i,j} \left( \frac{f_i^\eps(x) - f_k(x)}{\eps^2}\right)  \pi^\eps_{x, k} \pi^\eps_{x, i} \rho(x)$. Notice that from \eqref{eq:diff-2} and \eqref{eq:diff-3}, for $k \neq i,j$, $\pi_{x^-, k}^\eps \leq e^{-\tilde{\gamma}/\eps}$. This gives the bound
\begin{align}
\label{eq:bound-on-s-x-}
    \abs{S(x^-)} \lesssim \frac{1}{\eps^2} e^{-\tilde{\gamma}/\eps}.
\end{align}
Similarly, we have 
\begin{align}
    \label{eq:expr-M-x+}
    g_i^\eps(x^+) &= \frac{t}{\eps^2} \pi^\eps_{x^+, j} \pi^\eps_{x^+, i} \rho(x^+) + S(x^+),
\end{align}
where $S(x^+)$ verifies
\begin{align}
\label{eq:bound-on-s-x+}
    \abs{S(x^+)} \lesssim \frac{1}{\eps^2} e^{-\tilde{\gamma}/\eps}.
\end{align}
From equations \eqref{eq:expr-M-x-} and \eqref{eq:expr-M-x+}, we thus have the control
\begin{align}
\label{eq:sum-gneg_gpos}
    \abs{g_i^\eps(x^-) + g_i^\eps(x^+)} &\leq \frac{t}{\eps^2} \abs{\pi^\eps_{x^+, j} \pi^\eps_{x^+, i}\rho(x^+) - \pi^\eps_{x^-, j} \pi^\eps_{x^-, i}\rho(x^-) } + \abs{S(x^-)} + \abs{S(x^+)} \notag \\
    &\leq \frac{t}{\eps^2} \abs{\pi^\eps_{x^+, j} \pi^\eps_{x^+, i} - \pi^\eps_{x^-, j} \pi^\eps_{x^-, i} }\rho(x^+) + \frac{t}{\eps^2}\pi^\eps_{x^-, j} \pi^\eps_{x^-, i}\abs{\rho(x^+)-\rho(x^-)}  \\
    &\quad + \abs{S(x^-)} + \abs{S(x^+)} \notag
\end{align}

Now, notice that
\begin{align*}
    \pi^\eps_{x^+, i} &= \frac{ \exp \left(\frac{f_i^\eps(x^+)}{\eps}\right) }{ \exp \left(\frac{f_i^\eps(x^+)}{\eps}\right) + \exp \left(\frac{f_j^\eps(x^+)}{\eps}\right) + \sum_{k \neq i,j} \exp \left(\frac{f_k^\eps(x^+)}{\eps}\right) } \\
    &= \frac{1}{1 + e^{-t/\eps} + S_i(x^+)},
\end{align*}
where $S_\ell (x) = \sum_{k \neq i,j} \exp\left(\frac{f_k^\eps(x) - f_\ell^\eps(x)}{\eps}\right)$.
Similarly, we have
\begin{align*}
    \pi^\eps_{x^+, j} &= \frac{1}{1 + e^{t/\eps} + S_j(x^+)}, \\
    \pi^\eps_{x^-, i} &= \frac{1}{1 + e^{t/\eps} + S_i(x^-)}, \\
    \pi^\eps_{x^-, j} &= \frac{1}{1 + e^{-t/\eps} + S_j(x^-)}.
\end{align*}
Moreover, remark that for $\ell \in \{i, j\}$ and $x \in \{x^-, x^+\}$ we have the bound
$$ S_\ell (x) \lesssim e^{-\tilde{\gamma}/\eps}. $$
From these expressions we deduce the following bound:
\begin{align}
\label{eq:bound-first-term-gsum}
    \abs{\pi^\eps_{x^+, j} \pi^\eps_{x^+, i} - \pi^\eps_{x^-, j} \pi^\eps_{x^-, i} } \lesssim e^{-\tilde{\gamma}/\eps} e^{t/\eps}.
\end{align}
Now using that $\rho$ is $\alpha$-Hölder, we know that there exists a constant $C_\rho>0$ such that
\begin{align}
    \label{eq:bound-second-term-gsum}
    \abs{\rho(x^+) - \rho(x^-)} &\leq C_\rho\nr{x^+ - x^-}^\alpha = C_\rho \nr{2 t d_{ij}}^\alpha \leq C_\rho \left(\frac{2}{\delta}\right)^\alpha t^\alpha.
\end{align}
Plugging the bounds \eqref{eq:bound-on-s-x-}, \eqref{eq:bound-on-s-x+}, \eqref{eq:bound-first-term-gsum} and \eqref{eq:bound-second-term-gsum} into \eqref{eq:sum-gneg_gpos} then yields:
\begin{align*}
    \abs{g_i^\eps(x^+) + g_i^\eps(x^-)} &\lesssim \frac{1}{\eps^2} \left( t e^{-\tilde{\gamma}/\eps} e^{t/\eps} + t^{1+\alpha} + e^{-\tilde{\gamma}/\eps} \right).
\end{align*}
Injecting these bounds into integral \eqref{eq:integral-A} entails
\begin{align}
\label{eq:control-integral-A}
    \abs{\int_{A_{i, \eta, \gamma}^\eps} \frac{1}{\eps} \sum_{j=1, j \neq i}^N \left( \frac{f^\eps_i(x) - f^\eps_j(x)}{\eps^2} \right) \pi^\eps_{x,j} \pi^\eps_{x,i} \rho(x) dx} &\lesssim \int_0^{\eta \diam(\Y)} \frac{1}{\eps^2} \left( t e^{-\tilde{\gamma}/\eps} e^{t/\eps} + t^{1+\alpha} + e^{-\tilde{\gamma}/\eps} \right) \dd t \notag \\
    &\lesssim \frac{\eta^{2+\alpha}}{\eps^2} + \frac{1}{\eps^2}e^{-\tilde{\gamma}/\eps}\left(\eta + \eps\eta e^{\eta/\eps} - \eps^2(e^{\eta/\eps} - 1)\right). 
\end{align}

\textbf{Control on $B_{i, \eta, \gamma}^\eps$.}
We first derive a uniform bound on the integrand $$\sum_{j=1, j \neq i}^N \left( \frac{f^\eps_i(x) - f^\eps_j(x)}{\eps^2} \right) \pi^\eps_{x,j} \pi^\eps_{x,i}$$ on the domain $B_{i, \eta, \gamma}^\eps$, that is included in the $\eta$-neighborhood of $\Lag_i(\psi^\eps)$. For $x \in B_{i, \eta, \gamma}^\eps$, for any $j \neq i$, either $\abs{f_i^\eps(x) - f_j^\eps(x)} \leq \eta(\diam(\Y) + 1)$, and in this case
$$ \abs{ \frac{f^\eps_i(x) - f^\eps_j(x)}{\eps^2} } \pi^\eps_{x,j} \pi^\eps_{x,i} \lesssim \frac{\eta}{\eps^2}, $$
or $\abs{f_i^\eps(x) - f_j^\eps(x)} > \eta(\diam(\Y) +1)$, which entails $\pi_{x,j}^\eps \leq e^{-\eta/\eps}$. To see this, denote $k \in \arg \max_\ell f_\ell^\eps(x)$. Since $x$ is in a $\eta$-neighborhood of $\Lag_i(\psi^\eps)$, we have $$ 0 \leq f_k^\eps(x) - f_i^\eps(x) \leq \eta \nr{y_k - y_i} \leq \eta \diam(\Y).$$ Hence
\begin{align*}
    \abs{f_j^\eps(x) - f_k^\eps(x)} &\geq \abs{ \abs{f_j^\eps(x) - f_i^\eps(x)} - \abs{f_k^\eps(x) - f_i^\eps(x)} } \geq \eta.
\end{align*}
The inequality $\pi_{x,j}^\eps \leq e^{-\eta/\eps}$ then comes from the fact that $\pi_{x,j}^\eps \leq \exp\left(\frac{f_j^\eps(x) - f_k^\eps(x)}{\eps}\right)$. From these remarks we can therefore write for $x \in B_{i, \eta, \gamma}^\eps$:
\begin{align*}
    \sum_{j=1, j \neq i}^N \left( \frac{f^\eps_i(x) - f^\eps_j(x)}{\eps^2} \right) \pi^\eps_{x,j} \pi^\eps_{x,i} \lesssim \frac{1}{\eps^2}\left(\eta + e^{-\eta/\eps}\right).
\end{align*}
Finally, notice that $B_{i, \eta, \gamma}^\eps$ is made of a union of \textit{corners} of $\eta$-neighborhoods $\Lag_i(\psi^\eps)$, where \textit{corner} is meant for intersection of 2 hyperplanes. There are at most $\binom{N-1}{2} \leq N^2$ such corners. 
Denote  $\theta = \arg \max_{i, j,k \vert \angle y_i y_j y_k < \pi} \angle y_i y_j y_k $, i.e. the maximum angle that can be formed from a triplet of points in the support of the target that do not lie on a same line. Then the \emph{corners} that constitute $B_{i, \eta, \gamma}^\eps$ are actually included in \textit{cylinders} of \textit{length} at most $\diam(\X)$ and of radius at most $\frac{2 \gamma}{\cos(\theta/2)}$, that is of volume at most $\frac{4 \pi \diam(\X)^{d-2}}{\cos(\theta/2)^2} \gamma^2$. All these considerations allow us to write the following bound:
\begin{align}
\label{eq:control-integral-B}
    \abs{\int_{B_{i, \eta, \gamma}^\eps} \frac{1}{\eps} \sum_{j=1, j \neq i}^N \left( \frac{f^\eps_i(x) - f^\eps_j(x)}{\eps^2} \right) \pi^\eps_{x,j} \pi^\eps_{x,i} \rho(x) dx} \lesssim \frac{\gamma^2}{\eps^2} \left(\eta + e^{-\eta/\eps} \right).
\end{align}

\textbf{Conclusion.}
Finally, using equations \eqref{eq:control-partition-1}, \eqref{eq:control-partition-2}, \eqref{eq:control-integral-A}, \eqref{eq:control-integral-B} we get the wanted control:

\begin{align*}
    \abs{[\frac{\partial}{\partial \eps}(\nabla \Kant^\eps)(\psi)]_i} &= \abs{\int_{\X} \frac{1}{\eps} \sum_{j=1, j \neq i}^N \left( \frac{f^\eps_i(x) - f^\eps_j(x)}{\eps^2} \right) \pi^\eps_{x,j} \pi^\eps_{x,i} \rho(x) dx}\\
    &\lesssim \frac{1}{\eps^2} e^{-\eta /\eps} + \frac{\eta^{2+\alpha}}{\eps^2} + \frac{\gamma^2}{\eps^2} \left(\eta + e^{-\eta/\eps} \right) + \frac{1}{\eps^2}e^{-\tilde{\gamma}/\eps}\left(\eta + \eps\eta e^{\eta/\eps} - \eps^2(e^{\eta/\eps} - 1)\right).
\end{align*}

\end{proof}

\subsection{Proof of Theorem \ref{th:control-2nd-term}}
\label{sec:proof-th-control-2nd-term}
\begin{proof}
Here we assume that $\eps \leq 1$. Proposition \ref{prop:control-2nd-term-partition} entails the following inequality for any $\eta, \gamma > 0$:
\begin{align*}
    \abs{[\frac{\partial}{\partial \eps}(\nabla \Kant^\eps)(\psi)]_i} &\lesssim \frac{1}{\eps^2} e^{-\eta /\eps} + \frac{\eta^{2+\alpha}}{\eps^2} + \frac{\gamma^2}{\eps^2} \left(\eta + e^{-\eta/\eps} \right) + \frac{1}{\eps^2}e^{-\tilde{\gamma}/\eps}\left(\eta + \eps\eta e^{\eta/\eps} - \eps^2(e^{\eta/\eps} - 1)\right).
\end{align*}
We recall that $\tilde{\gamma} =\gamma \delta - \frac{\diam(\Y)^2}{\delta} \eta$. We choose $\gamma = \frac{\eta}{\delta} \left(\frac{\diam(\Y)^2}{\delta} + 2\right)$, which yields
\begin{align*}
    \abs{[\frac{\partial}{\partial \eps}(\nabla \Kant^\eps)(\psi)]_i} &\lesssim \frac{\eta^{2+\alpha} + \eta^3}{\eps^2} + \frac{e^{-\eta/\eps}}{\eps^2} \left( 1 + \eta^2 + \eps \eta - \eps^2 + (\eta + \eps^2)e^{-\eta/\eps} \right).
\end{align*}
Then, for any $\beta \in (\frac{2}{2+\alpha}, 1)$, choosing $\eta = \eps^\beta$ yields
\begin{align*}
    \abs{[\frac{\partial}{\partial \eps}(\nabla \Kant^\eps)(\psi)]_i} &\lesssim \eps^{(2+\alpha)\beta - 2} + \frac{e^{-1/\eps^{1-\beta}}}{\eps^2}.
\end{align*}
With $\alpha' = (2+\alpha)\beta - 2$, we get that for any $\alpha' \in (0, \alpha)$,
\begin{align*}
    \abs{[\frac{\partial}{\partial \eps}(\nabla \Kant^\eps)(\psi)]_i} &\lesssim \eps^{\alpha'} + \frac{e^{-1/\eps^{\frac{\alpha-\alpha'}{2+\alpha}}}}{\eps^2} \lesssim \eps^{\alpha'}.
\end{align*}
\end{proof}

\subsection{Convergence of $\psi^\eps$ to $\psi^0$ as $\eps$ goes to $0$}
\label{sec:cv-potentials}

\begin{proposition}
\label{prop:psi-cvg-eps-0}
The solutions $\psi^\eps$ to problem \eqref{eq:dual} verify:
$$ \lim_{\eps \to 0} \psi^\eps = \psi^0. $$
\end{proposition}

\begin{proof}
Let's first prove that for any $\eps \geq 0$, the solution $\psi^\eps$ to problem \eqref{eq:dual} verifies:
$$ \nr{\psi^\eps}_\infty \leq R_\X \diam(\Y) + \eps \log(1/\underline{\mu}).$$
For $\eps>0$, recall that the first order condition \eqref{eq:first-order-cond} entails that $$e^{\frac{\psi^\eps_i}{\eps}} \mu_i = \int_\X e^{\frac{\sca{x}{y_i} - (\psi^\eps)^{c,\eps}}{\eps}} \dd \rho(x).$$
Thus right-hand sign of this equality can be seen as a $R_\X$-Lipschitz function of $y_i$. Therefore we have the bound:
\begin{equation}
    \label{eq:diff-psi-eps-i-psi-eps-j}
    \abs{\psi^\eps_i - \psi^\eps_j} \leq R_\X\abs{y_i - y_j} + \eps\abs{\log(\frac{\mu_i}{\mu_j})} \leq R_\X \diam(\Y) + \eps \log(1/\underline{\mu}).
\end{equation}
Now recall that $\sca{\psi^\eps}{\mathbb{1}_N} = 0$, which means that the components of $\psi^\eps \in \Rsp^N$ take both positive and negative values.
This entails for any $i \in \{1, \dots, N\}$:
\begin{align*}
    \abs{\psi^\eps_i} = \abs{\psi^\eps_i - 0} \leq \max_j \abs{\psi^\eps_i - \psi^\eps_j} \leq R_\X \diam(\Y) + \eps \log(1/\underline{\mu}).
\end{align*}
When $\eps=0$, the bound \eqref{eq:diff-psi-eps-i-psi-eps-j} comes from the fact that $\psi^0$ is a $R_\X$-Lipschitz function from $\Y$ to $\Rsp$. Indeed, Proposition 1.11 in \cite{santambrogio2015optimal} ensures that for any $i \in \{1, \dots, N\}$, 
$$\psi^0_i = \psi^0(y_i) = (\psi^0)^{**}(y_i) = \sup_{x\in \X} \sca{x}{y_i} - (\psi^0)^*(x),$$  which reads as a $R_\X$-Lipschitz function of $y_i$. We conclude similarly to the case $\eps>0$ to ensure $\nr{\psi^0}_\infty \leq R_\X \diam(\Y)$.

Now consider a sequence $(\eps_k)_k > 0$ such that $\lim_{k \to \infty} \eps_k = 0$. By what precedes, the sequence $(\psi^{\eps_k})_k$ is bounded and one can extract a converging subsequence (that we do not relabel). Notice now that for any $x \in \X$ and $\psi \in \Rsp^N$, the $(c,\eps)$-transform $\psi^{c,\eps}(x)$ corresponds to a rescaled LogSumExp (or smooth maximum) of the vector $(\sca{x}{y_i} - \psi_i)_{i=1, \dots, N}$:
$$ \psi^{c,\eps}(x) = \eps LSE \left( \frac{(\sca{x}{y_i} - \psi_i)_{i=1, \dots, N}}{\eps} \right), $$
where $LSE(z_1, \dots, z_N) = \log(\exp(z_1) + \dots \exp(z_N))$. Bounds on $LSE$ allow us to write that for any, $x \in \X$ and $\eps>0$ we have $$\psi^*(x) \leq \psi^{c, \eps}(x) \leq \psi^*(x) + \eps \log N,$$
where we recall that $\psi^*(x) = \max_{i=1,\dots,N} \sca{x}{y_i} - \psi_i$ denotes the Legendre transform of $\psi$ evaluated in $x$.
Thus if we consider $k \in \Nsp$, we have by optimiality of $\psi^0$, $\psi^{\eps_k}$ for their respective problems the inequalities:
\begin{align*}
    \sca{(\psi^0)^*}{\rho} + \sca{\psi^0}{\mu} &\leq \sca{(\psi^{\eps_k})^*}{\rho} + \sca{\psi^{\eps_k}}{\mu} \\
    &\leq \sca{(\psi^{\eps_k})^{c,{\eps_k}}}{\rho} + \sca{\psi^{\eps_k}}{\mu} + \eps_k \\
    &\leq \sca{(\psi^0)^{c,{\eps_k}}}{\rho} + \sca{\psi^0}{\mu} + \eps_k \\
    &\leq \sca{(\psi^0)^{*}}{\rho} + \sca{\psi^0}{\mu} + \eps_k(1 + \log N).
\end{align*}
Hence we have the limit:
$$ \lim_{k \to \infty} \sca{(\psi^{\eps_k})^*}{\rho} + \sca{\psi^{\eps_k}}{\mu} = \sca{(\psi^0)^*}{\rho} + \sca{\psi^0}{\mu}. $$
By unicity of the solution of the unregularized problem on $(\mathbb{1}_N)^\perp$, this proves that
$$ \lim_{k \to \infty} \psi^{\eps_k} = \psi^0. $$
There is thus only one accumulation point for the bounded sequence $(\psi^{\eps_k})_k$, which shows that the whole sequence converges to this point. 
\end{proof}

\subsection{Proof of Corollary \ref{coro:stab-potentials-eps}}
\label{sec:proof-cor-scaling}
\begin{proof}
Following Theorem \ref{th:control-dot-psi}, let $C>0$ (depending on $\X, \rho, \Y, \mu$) be such that $\nr{\dot{\psi^\eta}}_2 \leq C \tau^{\alpha'}$ for $\tau \in [\eps', \eps]$.
Then notice
\begin{align*}
    \nr{\psi^\eps - \psi^{\eps'}}_\infty \leq \nr{\psi^\eps - \psi^{\eps'}}_2 = \nr{\int_{\eps'}^\eps \dot{\psi^\tau} \dd \tau}_2 \leq \int_{\eps'}^\eps \nr{\dot{\psi^\tau}}_2 \dd \tau \leq C \eps^{\alpha'}(\eps - \eps').
\end{align*}
Letting $\eps'$ go to $0$ and using Proposition \ref{prop:psi-cvg-eps-0} yields
\begin{align*}
    \nr{\psi^\eps - \psi^0}_\infty \leq C\eps^{1+\alpha'}.
\end{align*}
For the second result, we use
\begin{align*}
    \nr{(\psi^\eps)^{c,\eps} - (\psi^0)^*}_\infty \leq \nr{(\psi^\eps)^{c,\eps} - (\psi^0)^{c,\eps}}_\infty + \nr{(\psi^0)^{c,\eps} - (\psi^0)^*}_\infty 
\end{align*}
One can easily show with the definition of the $(c,\eps)$-transform that $\nr{\psi^\eps - \psi^0}_\infty \leq \frac{C}{1 + \alpha'} \eps^{1+\alpha'}$ entails
$$ \nr{(\psi^\eps)^{c,\eps} - (\psi^0)^{c,\eps}}_\infty \leq \frac{C}{1 + \alpha'} \eps^{1+\alpha'}. $$
On the other hand, $\nr{(\psi^0)^{c,\eps} - (\psi^0)^*}_\infty \leq \eps \log N$ is a LogSumExp property. This property can be refined to get to the third result: we have for all $x \in \X$
\begin{align*}
    (\psi^0)^*(x) \leq (\psi^0)^{c,\eps}(x) &= \eps \log\left( \sum_{j=1}^N e^{\frac{\sca{x}{y_j}-\psi^0_j}{\eps}} \right) \\
    &= (\psi^0)^*(x) + \eps \log\left( \sum_{j=1}^N e^{\frac{\sca{x}{y_j}-\psi^0_j - (\psi^0)^*(x)}{\eps}} \right).
\end{align*}
But for $\rho$-almost every $x \in \X$, there is only one $i \in \{1, \dots, N\}$ that satisfies $(\psi^0)^*(x) = \sca{x}{y_i} - \psi^0_i$. Thus for such $x$, denoting $c_x = \min_{j \neq i} (\sca{x}{y_i} - \psi^0_i) - (\sca{x}{y_j} - \psi^0_j) > 0$, we have
\begin{align*}
    \sum_{j=1}^N e^{\frac{\sca{x}{y_j}-\psi^0_j - (\psi^0)^*(x)}{\eps}} \leq 1 + (N-1)e^{-c_x/\eps}.
\end{align*}
We thus get:
\begin{align*}
    (\psi^0)^*(x) \leq (\psi^0)^{c,\eps}(x) \leq (\psi^0)^*(x) + (N-1) \eps e^{-c_x/\eps}.
\end{align*}
From this we deduce that for $\rho$-a.e. $x \in \X$,
\begin{align*}
    \abs{(\psi^0)^{c,\eps}(x) - (\psi^0)^*(x)} \lesssim \eps e^{-c_x/\eps} \lesssim \eps^{1+\alpha'}.
\end{align*}
Finally, we use the notation of Section \ref{sec:bound_2nd_term} that denotes
$$ \frac{ \dd \pi^\eps }{\dd \rho \otimes \sigma} (x, y_i) = \pi^\eps_{x,i}.$$
Notice that for any $i \in \{1, \dots, N\}$, 
$$
    \pi^0_{x,i} =\left\{
    \begin{array}{ll}
        1 & \mbox{if } \sca{x}{y_i} - \psi^0_i \geq  \sca{x}{y_j} - \psi^0_j \quad \forall j,\\
        0 & \mbox{else.}
    \end{array}
\right.
$$
If $\pi^0_{x,i}=0$, then with the same $c_x$ as before (assuming that only one $i \in \{1, \dots, N\}$ satisfies $(\psi^0)^*(x) = \sca{x}{y_i} - \psi^0_i$),
\begin{align*}
    \abs{\pi^\eps_{x,i} - \pi^0_{x,i}} &= \pi^\eps_{x,i} \\
    &=\frac{ e^{ \frac{\sca{x}{y_i} - \psi^\eps_i}{\eps} } }{ \sum_j e^{\frac{\sca{x}{y_j} - \psi^\eps_j}{\eps} } } \\
    &\leq e^{2C\eps^{\alpha'}/(1+\alpha')} \frac{ e^{ \frac{\sca{x}{y_i} - \psi^0_i}{\eps} } }{ \sum_j e^{\frac{\sca{x}{y_j} - \psi^0_j}{\eps} } } \\
     &\leq e^{2C\eps^{\alpha'}/(1+\alpha')} e^{-c_x / \eps}\\
     &\lesssim e^{-c_x / \eps}.
\end{align*}
If $\pi^0_{x,i}=1$, then
\begin{align*}
    \abs{\pi^\eps_{x,i} - \pi^0_{x,i}} &= \abs{\pi^\eps_{x,i} - 1}\\
    &=\frac{  \sum_{j \neq i} e^{\frac{\sca{x}{y_j} - \psi^\eps_j}{\eps}}  }{ \sum_j e^{\frac{\sca{x}{y_j} - \psi^\eps_j}{\eps} } } \\
     &\leq e^{2C\eps^{\alpha'}/(1+\alpha')} \frac{  \sum_{j \neq i} e^{\frac{\sca{x}{y_j} - \psi^0_j}{\eps}}  }{ \sum_j e^{\frac{\sca{x}{y_j} - \psi^0_j}{\eps} } } \\
     &\leq e^{2C\eps^{\alpha'}/(1+\alpha')} \sum_{j \neq i} e^{\frac{\sca{x}{y_j} - \psi^0_j - (\psi^0)^*(x)}{\eps}} \\
     &\leq e^{2C\eps^{\alpha'}/(1+\alpha')} (N-1) e^{-c_x/\eps}\\
     &\lesssim e^{-c_x / \eps}.
\end{align*}
This proves that for $\rho$-a.e. $x \in \X$ and all $i \in \{1, \dots, N\}$, $\abs{\pi^\eps_{x,i} - \pi^0_{x, i}} \lesssim e^{-c_x/\eps}$.
\end{proof}

\subsection{Expansion of the Difference of Costs (Theorem \ref{th:suboptimality-cv})}
\label{sec:asymptotics}
The proof of Theorem \ref{th:suboptimality-cv} follows very closely the proof of Theorem 1.1 in \cite{altschuler2021asymptotics} and we thus make numerous mentions of results from this paper.
\begin{proof}
Let $\eps \in (0, 1]$ and recall that $\Wass_{2,\eps}^2(\rho, \mu)$ is computed using the solution of the regularized maximum correlation problem \eqref{eq:primal} with regularization parameter $\frac{\eps}{2}$:
$$ \Wass_{2,\eps}^2(\rho, \mu) = \Esp_{(x, y) \sim  \pi^{\eps/2}} \nr{x - y}^2. $$
Now notice that as in Lemma 5.2 in \cite{altschuler2021asymptotics} we can write by strong duality
\begin{align*}
    \Esp_{\pi^0} \sca{x}{y} &= \Esp_{x \sim \rho} (\psi^0)^*(x) + \Esp_{y \sim \mu} \psi^0(y) \\
    &= \Esp_{(x, y) \sim \pi^{\eps/2}} \left( (\psi^0)^*(x) + \psi^0(y) \right).
\end{align*}
Therefore the difference of costs reads
\begin{align*}
    \Wass_{2,\eps}^2(\rho, \mu) - \Wass_{2}^2(\rho, \mu) &= \Esp_{\pi^{\eps/2}} \nr{x - y}^2 - \Esp_{\pi^0} \nr{x - y}^2 \\
    &= 2 \left( \Esp_{\pi^0} \sca{x}{y} - \Esp_{\pi^{\eps/2}} \sca{x}{y} \right) \\
    &= 2 \Esp_{(x, y) \sim \pi^{\eps/2}} \left( (\psi^0)^*(x) + \psi^0(y) - \sca{x}{y} \right) \\
    &= 2 \sum_{i, j} \int_{\Lag_i(\psi^0)} \left( (\psi^0)^*(x) + \psi^0_j - \sca{x}{y_j} \right) \pi^{\eps/2}_{x, j} \dd \rho(x) \\
    &= \sum_{i, j} \int_{\Lag_i(\psi^0)} \Delta_{ij}(x) \pi^{\eps/2}_{x, j} \dd \rho(x),
\end{align*}
where we denoted for $x \in \Lag_i(\psi^0)$,
$$\Delta_{ij}(x) =  2 \left((\psi^0)^*(x) + \psi^0_j - \sca{x}{y_j} \right)= 2\left(\sca{x}{y_i - y_j} - \psi^0_i + \psi^0_j\right) \geq 0.$$ Using Theorem \ref{th:control-dot-psi} and its Corollary \ref{coro:stab-potentials-eps} we know that there exists $C>0$ depending on $\X, \rho, \Y, \mu$ such that for $\alpha' \in (0, \alpha)$,
$$ \nr{\psi^{\eps/2} - \psi^0}_\infty \leq C \left(\frac{\eps}{2}\right)^{1+\alpha'}. $$
From this bound, using that $\pi^{\eps/2}_{x,j} = \frac{\exp\left(\frac{\sca{x}{y_j} - \psi^{\eps/2}_j}{\eps/2}\right)}{\sum_\ell \exp\left(\frac{\sca{x}{y_\ell} - \psi^{\eps/2}_\ell}{\eps/2}\right)}$ we deduce the bounds
$$ e^{-2C(\eps/2)^{\alpha'}} \frac{e^{-\Delta_{ij}(x)/\eps}}{\sum_\ell e^{-\Delta_{i\ell}(x)/\eps}} \leq \pi^{\eps/2}_{x, j} \leq  e^{2C(\eps/2)^{\alpha'}} \frac{e^{-\Delta_{ij}(x)/\eps}}{\sum_\ell e^{-\Delta_{i\ell}(x)/\eps}}. $$
Hence we deduce the following control:
\begin{align*}
\Bigg\lvert \Wass_{2,\eps}^2(\rho, \mu) - &\Wass_{2}^2(\rho, \mu) - \sum_{i, j} \int_{\Lag_i(\psi^0)} \Delta_{ij}(x) \frac{e^{-\Delta_{ij}(x)/\eps}}{\sum_\ell e^{-\Delta_{i\ell}(x)/\eps}} \dd \rho(x) \Bigg\rvert \\
&\lesssim \eps^{\alpha'} \sum_{i, j} \int_{\Lag_i(\psi^0)} \Delta_{ij}(x) \frac{e^{-\Delta_{ij}(x)/\eps}}{\sum_\ell e^{-\Delta_{i\ell}(x)/\eps}} \dd \rho(x). 
\end{align*}
Lemma 6.4 found in \cite{altschuler2021asymptotics} then asserts that
\begin{align*}
     \sum_{i, j} \int_{\Lag_i(\psi^0)} \Delta_{ij}(x) \frac{e^{-\Delta_{ij}(x)/\eps}}{\sum_\ell e^{-\Delta_{i\ell}(x)/\eps}} \dd \rho(x) = \eps^2 \frac{\pi^2}{12} \sum_{i<j} \frac{w_{ij}}{\nr{y_i - y_j}} + o(\eps^2).
\end{align*}
Directly injecting this into the last control would lead to the exact same asymptotic result as the one given in Theorem 1.1 in \cite{altschuler2021asymptotics}. We refine this last asymptotic development to a non-asymptotic one by leveraging the fact that we assumed the source $\rho$ to be $\alpha$-Hölder continuous.

For any $i, j \in \{1, \dots, N\}$, denote $I_{ij} =  \int_{\Lag_i(\psi^0)} \Delta_{ij}(x) \frac{e^{-\Delta_{ij}(x)/\eps}}{\sum_\ell e^{-\Delta_{i\ell}(x)/\eps}} \dd \rho(x) $.
Then notice that for $i \neq j$,
\begin{align}
\label{eq:upper-bound-Iij}
    I_{ij} \leq \int_{\Lag_i(\psi^0)} \Delta_{ij}(x) \frac{e^{-\Delta_{ij}(x)/\eps}}{1 + e^{-\Delta_{ij}(x)/\eps}} \dd \rho(x)
\end{align}
Similarly to \cite{altschuler2021asymptotics}, introduce for $a > 0$
$$ S_{ij}(a) = \{x \in \Lag_i(\psi^0) \vert \Delta_{ik}(x) \geq a \forall k \neq i,j\}.$$
Notice that $S_{ij}(0) = \Lag_i(\psi^0)$, and that for some $a>0$:
\begin{align}
\label{eq:lower-bound-Iij}
    I_{ij} &\geq  \int_{S_{ij}(a)} \Delta_{ij}(x) \frac{e^{-\Delta_{ij}(x)/\eps}}{1 + (N-2)e^{-a/\eps} + e^{-\Delta_{ij}(x)/\eps}} \dd \rho(x) \notag \\
    &\geq  \int_{S_{ij}(a)} \Delta_{ij}(x) \frac{e^{-\Delta_{ij}(x)/\eps}}{c(a) + e^{-\Delta_{ij}(x)/\eps}} \dd \rho(x),
\end{align}
for $c(a) = 1 + (N-2)e^{-a/\eps}> 1$. The quantity $I_{ij}$ is thus bounded by integrals of the form $$ \int_{S_{ij}(a)} \Delta_{ij}(x) \frac{e^{-\Delta_{ij}(x)/\eps}}{c + e^{-\Delta_{ij}(x)/\eps}} \dd \rho(x)$$ for some $a \geq 0$ and $c \geq 1$. Let's find a non-asymptotic control of such integrals in terms of $\eps$.

Recall that for $x \in \Lag_i(\psi^0)$, $\Delta_{ij}(x) = 2\left(\sca{x}{y_i - y_j} - \psi^0_i + \psi^0_j\right)$. The coarea formula then ensures:
\begin{align}
\label{eq:integral-to-bound}
    \int_{S_{ij}(a)} \Delta_{ij}(x) \frac{e^{-\Delta_{ij}(x)/\eps}}{c + e^{-\Delta_{ij}(x)/\eps}} \dd \rho(x) 
    &= \frac{1}{2  \nr{y_i-y_j}} \int_0^{\infty}  t \frac{e^{-t/\eps}}{c + e^{-t/\eps}} h_{ij}(t; a) \dd t 
\end{align}
where we denoted $h_{ij}(t; a) = \int_{S_{ij}(a)\cap(\Delta_{ij})^{-1}(t)} \rho(x) \dd\mathcal{H}^{d-1}(x)$ similarly to \cite{altschuler2021asymptotics} (one can already notice that $h_{ij}(0; 0) = w_{ij}$).
Notice then from Lemma 6.2 in \cite{altschuler2021asymptotics} that
\begin{align*}
    \Bigg\lvert \int_0^\infty  t \frac{e^{-t/\eps}}{c + e^{-t/\eps}} h_{ij}(t; a) \dd t - &\eps^2 h_{ij}(0; a) \left( -\Li_2(-1/c) \right) \Bigg\rvert = \eps^2 \abs{\int_0^{\infty}  u \frac{e^{-u}}{c + e^{-u}} (h_{ij}(\eps u; a) - h_{ij}(0; a)) \dd u },
\end{align*}
where $\Li_2$ denotes the dilogarithm function.
We now focus on the difference $h_{ij}(\eps u; a) - h_{ij}(0; a)$:
\begin{align*}
    h_{ij}(\eps u; a) - h_{ij}(0; a) = \int_{S_{ij}(a)\cap(\Delta_{ij})^{-1}(\eps u)} \rho(x) \dd\mathcal{H}^{d-1}(x) - \int_{S_{ij}(a)\cap(\Delta_{ij})^{-1}(0)} \rho(x) \dd\mathcal{H}^{d-1}(x)
\end{align*}
One can notice that there exists a set $R_{ij}(\eps u)$, that is a subset of an hyperplane, and whose $(d-1)$-area is (at most) linear in $\eps u$, such that either
$$ S_{ij}(a)\cap(\Delta_{ij})^{-1}(\eps u) = \left( S_{ij}(a)\cap(\Delta_{ij})^{-1}(0) + (\eps u) n_{ij} \right) \cup R_{ij}(u), $$
or
$$ S_{ij}(a)\cap(\Delta_{ij})^{-1}(\eps u) = \left( S_{ij}(a)\cap(\Delta_{ij})^{-1}(0) + (\eps u) n_{ij} \right) \setminus R_{ij}(u), $$
where $n_{ij} = \frac{y_i - y_j}{\nr{y_i - y_j}}$.
Hence we have:
\begin{align*}
    \abs{h_{ij}(\eps u; a) - h_{ij}(0; a)} \leq \int_{S_{ij}(a)\cap(\Delta_{ij})^{-1}(0)}\abs{ \rho(x + (\eps u) n_{ij}) - \rho(x)} \dd\mathcal{H}^{d-1}(x)+ \int_{R_{ij}(\eps u)} \dd\mathcal{H}^{d-1}(x)
\end{align*}
Recalling that $\rho$ is $\alpha$-Hölder continuous, we have $\abs{\rho(x + (\eps u) n_{ij}) - \rho(x)} \lesssim \eps^\alpha u^\alpha$.
Hence we deduce 
\begin{align*}
    \abs{h_{ij}(\eps u; a) - h_{ij}(0; a)} \lesssim \eps^\alpha u^\alpha + \eps u.
\end{align*}
We thus have shown that
\begin{align*}
    \Bigg\lvert \int_0^\infty  t \frac{e^{-t/\eps}}{c + e^{-t/\eps}} h_{ij}(t; a) \dd t - \eps^2 h_{ij}(0; a) \left( -\Li_2(-1/c) \right) \Bigg\rvert &\lesssim \eps^2 \int_0^{\infty}  u \frac{e^{-u}}{c + e^{-u}} (\eps^\alpha u^\alpha + \eps u) \dd u \\
    &\lesssim \eps^{2+\alpha},
\end{align*}
where we used that $c\geq1$ and $\eps \leq 1$.

We finally bound the distance between $h_{ij}(0; a) \left( -\Li_2(-1/c) \right)$ and $h_{ij}(0; 0) \left( -\Li_2(-1) \right)$. We have the following inequality:
\begin{align*}
    \abs{h_{ij}(0; a) \left( -\Li_2(-1/c) \right) - h_{ij}(0; 0) \left( -\Li_2(-1) \right)} &\leq \abs{-\Li_2(-1)}\abs{h_{ij}(0; a) - h_{ij}(0; 0)}\\
    &+ \abs{h_{ij}(0; a)} \abs{-\Li_2(-1/c) - (-\Li_2(-1))}
\end{align*}
The quantity $\abs{h_{ij}(0; a) - h_{ij}(0; 0)}$ obviously scales linearly with $a$. Then one can notice that on $[1, c]$, the function $t \mapsto -\Li_2(-1/t)$ is $(-\Li_2(-1))$-Lipschitz. These facts ensure the following control:
\begin{align*}
    \abs{h_{ij}(0; a) \left( -\Li_2(-1/c) \right) - h_{ij}(0; 0) \left( -\Li_2(-1) \right)} &\lesssim a + \abs{c - 1}.
\end{align*}
This allows to write
\begin{align*}
    \Bigg\lvert \int_0^\infty  t \frac{e^{-t/\eps}}{c + e^{-t/\eps}} h_{ij}(t; a) \dd t - \eps^2 h_{ij}(0; 0) \left( -\Li_2(-1) \right) \Bigg\rvert &\lesssim  \eps^{2+\alpha} + \eps^2(a + \abs{c - 1}).
\end{align*}
Thus, setting $a = \eps^\alpha$, we get $c(a) = 1 + (N-2) e^{-1/\eps^{1-\alpha}}$ and

\begin{align*}
    \Bigg\lvert \int_0^\infty  t \frac{e^{-t/\eps}}{c(a) + e^{-t/\eps}} h_{ij}(t; a) \dd t - \eps^2 h_{ij}(0; 0) \left( -\Li_2(-1) \right) \Bigg\rvert &\lesssim  \eps^{2+\alpha}.
\end{align*}
This leads to:
\begin{align*}
    \abs{I_{ij} - \eps^2 \frac{h_{ij}(0; 0)}{2 \nr{y_i - y_j}} \left( -\Li_2(-1) \right)} = \abs{I_{ij} - \eps^2 \frac{w_{ij}}{\nr{y_i - y_j}} \frac{\pi^2}{24}} \lesssim \eps^{2 + \alpha}.
\end{align*}
Eventually, recalling the bound
\begin{align*}
\abs{ \Wass_{2,\eps}^2(\rho, \mu) - \Wass_{2}^2(\rho, \mu) - \sum_{i,j} I_{ij} } \lesssim \eps^{\alpha'} \sum_{i,j} I_{ij},
\end{align*}
we obtain the wanted control for $\eps \leq 1$:
\begin{align*}
\abs{ \Wass_{2,\eps}^2(\rho, \mu) - \Wass_{2}^2(\rho, \mu) - \eps^2 \frac{\pi^2}{12} \sum_{i<j} \frac{w_{ij}}{\nr{y_i - y_j}} } \lesssim \eps^{2+\alpha}. 
\end{align*}
\end{proof}

\subsection{Tightness of Theorem \ref{th:suboptimality-cv}}
\label{sec:tightness-th-suboptimality}

We now show that Theorem \ref{th:suboptimality-cv} is tight on a simple one-dimensional example. 
\begin{theorem}
In Theorem \ref{th:suboptimality-cv}, there exists $\rho, \mu$ such that for $\eps \leq 1$,
\begin{align*}
    \abs{ \Wass_{2,\eps}^2(\rho, \mu) - \Wass_{2}^2(\rho, \mu) - \eps^2 \frac{\pi^2}{12} \sum_{i<j} \frac{w_{ij}}{\nr{y_i - y_j}} } \gtrsim \eps^{2+\alpha}. 
\end{align*}
\end{theorem}
\begin{proof}
Once again, we rely on results from \cite{altschuler2021asymptotics} (Section 3), where the following formula for the difference of costs for the transport between a continuous symmetric density $\rho$ supported on $[-1, 1]$ and the target $\mu = \frac{1}{2}(\delta_{\{-1\}} + \delta_{\{+1\}})$ is given:
$$ \Wass_{2, \epsilon}^2(\rho, \mu) - \Wass_2^2(\rho, \mu) = 8 \int_0^1 \frac{x}{1 + e^{4x/\eps}} \rho(x) \dd x. $$
We consider, for $\alpha \in (0, 1]$ the following $\alpha$-Hölder density for the source:
$$ \rho(x) = \frac{1+\alpha}{2 \alpha} (1 - \abs{x}^\alpha) \mathbb{1}_{[-1, 1]}. $$ 
We can then derive the difference between the suboptimality and its asymptote for $\eps \leq 1$:
\begin{align*}
    \abs{ \Wass_{2, \epsilon}^2(\rho, \mu) - \Wass_2^2(\rho, \mu) - \frac{\pi^2 \rho(0)}{24} \eps^2 } &= \abs{ 8 \int_0^1 \frac{x}{1 + e^{4x/\eps}} \rho(x) \dd x - \rho(0) \frac{\eps^2}{2} (-\Li_2(-1)) } \\
    &\geq \abs{8 \int_0^1 \frac{x}{1 + e^{4x/\eps}} (\rho(x) -\rho(0)) \dd x} \\
    &\quad - \abs{ 8 \rho(0) \int_0^1 \frac{x}{1 + e^{4x/\eps}} \dd x - \rho(0) \frac{\eps^2}{2} (-\Li_2(-1)) } \\
    &= \abs{8 \frac{1+\alpha}{2 \alpha} \int_0^1 \frac{x^{1+\alpha}}{1 + e^{4x/\eps}} \dd x } - \abs{\rho(0) \frac{\eps^2}{2} \int_{4/\eps}^\infty \frac{t e^{-t}}{1 + e^{-t}} \dd t } \\
    &\geq \frac{4(1 + \alpha)}{\alpha} \left( \frac{\eps}{4}\right)^{2 + \alpha} \int_0^{4/\eps} \frac{t^{1+\alpha}}{1 + e^t} \dd t - 4 \rho(0) \eps e^{-4/\eps} \\
    &\geq \frac{4(1 + \alpha)}{\alpha} \left( \frac{\eps}{4}\right)^{2 + \alpha} \int_0^{4} \frac{t^{1+\alpha}}{1 + e^t} \dd t - 4 \rho(0) \eps e^{-4/\eps} \\
    &\geq \frac{2(1+\alpha)}{\alpha} \left( \frac{2}{(1 + e^4)(2 + \alpha)} \eps^{2+\alpha} - \eps e^{-4/\eps}) \right).
\end{align*}
Thus for $\eps$ small enough we get:
\begin{align*}
    \abs{ \Wass_{2, \epsilon}^2(\rho, \mu) - \Wass_2^2(\rho, \mu) - \frac{\pi^2 \rho(0)}{24} \eps^2 } \gtrsim \eps^{2+\alpha}.
\end{align*}
\end{proof}

\end{document}